\documentclass[sigconf,authorversion]{acmart}

\AtBeginDocument{%
  }

\copyrightyear{2026}
\acmYear{2026}
\setcopyright{cc}
\setcctype{by}
\acmConference[GECCO '26]{Genetic and Evolutionary Computation Conference}{July 13--17, 2026}{San Jose, Costa Rica}
\acmBooktitle{Genetic and Evolutionary Computation Conference (GECCO '26), July 13--17, 2026, San Jose, Costa Rica}
\acmDOI{10.1145/3795095.3805131}
\acmISBN{979-8-4007-2487-9/2026/07}

\usepackage{hyperref}[hyperfootnotes=true]
\usepackage{float}
\usepackage{physics}
\usepackage{tabularx}
    \newcolumntype{Y}{>{\centering\arraybackslash}X}
\usepackage{booktabs}
\usepackage{pdflscape}
\usepackage{csvsimple}
\usepackage{url}
\usepackage{nameref}
\usepackage{tablefootnote}
\usepackage{color, colortbl}
\usepackage{mathtools}
\usepackage{amsmath,amsfonts,mathtools}
\usepackage{bbm}
\usepackage[algo2e,ruled,vlined,linesnumbered]{algorithm2e}
\usepackage{xparse}
\usepackage{makecell,booktabs}
\usepackage{subfigure}
\usepackage{caption}
\newcommand{\ignore}[1]{}
\usepackage{thmtools,thm-restate}

\usepackage{tikz} 
\usetikzlibrary{shapes.geometric}
\usepackage{pgfplots}
\usepgfplotslibrary{fillbetween,statistics}
\pgfplotsset{compat=newest}
\pgfplotscreateplotcyclelist{tikzcycle}{%
draw=none \\ 
draw=none \\
fill=black!10 \\
thick,black \\ 
thick,blue,mark=square*\\ 
thick,red,mark=*\\ 
}
\pgfplotscreateplotcyclelist{muscycle}{%
thick,blue,mark=square*\\ 
thick,red,mark=*\\ 
}

\usepackage
[
    noabbrev,   
    nameinlink, 
]
{cleveref} 

\DeclareDocumentCommand{\set}{m g o}{
    \ensuremath{
        \IfNoValueTF{#3}{\left}{#3}\{#1
            \IfNoValueTF{#2}{}{
                \ \IfNoValueTF{#3}{\left}{#3}\vert\ \vphantom{#1}#2\IfNoValueTF{#3}{\right.}{}
            } \IfNoValueTF{#3}{\right}{#3}\}
    }\xspace
}

\definecolor{lightgreen}{rgb}{0.75,0.92,0.61}

\newcommand{\labelname}[1]{
  \def\@currentlabelname{#1}}

\newcommand{\N}{\mathbb{N}}

\newcommand{\ones}[1]{|#1|_1}
\newcommand{\zeros}[1]{|#1|_0}

\newcommand{\Bin}{\text{Bin}}
\definecolor{mygreen}{RGB}{1, 150, 122}

\providecommand{\ignore}[1]{}

\newcommand{\muoea}{($\mu$+1) EA\xspace}
\newcommand{\muoga}{($\mu$+1) GA\xspace}

\newcommand{\jump}{\mbox{\textsc{Jump$_k$}}\xspace}

\newcommand{\plateau}{\textsc{plateau}\xspace}

\newcommand{\pfu}{p_{\text{fu}}}

\newenvironment{proofofclaim}{\textsc{Proof of Claim.}}{\hfill\scriptsize\scalebox{1}{$\blacksquare$}}

	{
	\begin{center}
	\begin{algorithm2e}
	}%
	{
	\end{algorithm2e}
	\end{center}
	}
	
\allowdisplaybreaks
\allowdisplaybreaks

\begin{document}

\title[Parent Selection Mechanisms in Elitist Crossover-Based Algorithms]{Parent Selection Mechanisms\\in Elitist Crossover-Based Algorithms}

\author{Andre Opris}
\email{andre.opris@uni-passau.de}
\orcid{0000-0002-7730-7831}
\affiliation{%
  \institution{
University of Passau}
  \city{Passau}
  \country{Germany}
}

\author{Denis Antipov}
\email{denis.antipov@lip6.fr}
\orcid{0000-0001-7906-096X}
\affiliation{%
  \institution{Sorbonne Universit\'e, CNRS, LIP6}
  \city{Paris}
  \country{France}
}

\begin{abstract}  
    Parent selection methods are widely used in evolutionary computation to accelerate the optimization process, yet their theoretical benefits are still poorly understood. In this paper, we address this gap by proposing a parent selection strategy for the $(\mu+1)$ genetic algorithm (GA) that prioritizes the selection of maximally distant parents for crossover. We show that, with an appropriately chosen population size, the resulting algorithm solves the Jump$_k$ problem in $O(k4^kn\log(n))$ expected time. This bound is significantly smaller than the best known bound of $O(n\mu\log(\mu)+n\log(n)+n^{k-1})$ for any $(\mu+1)$~GA using no explicit diversity-preserving mechanism and a constant crossover probability. 
  
    To establish this result, we introduce a novel diversity metric that captures both the maximum distance between pairs of individuals in the population and the number of pairs achieving this distance. The main novelty of our analysis is that it relies on crossover as a mechanism for creating and maintaining diversity throughout the run, rather than using crossover only in the final step to combine already diversified individuals. The insights provided by our analysis contribute to a deeper theoretical understanding of the role of crossover in the population dynamics of genetic algorithms.
\end{abstract}



\maketitle

\section{Introduction} \label{sec:introduction}

\emph{Evolutionary algorithms} (EAs for brevity) are one of the most widely used random search heuristics for solving real-world optimization problems. They combine ease of application (they do not require any prior knowledge of the problem and treat it as a black box) with the ability to efficiently find good solutions.  During the last three decades theory has derived lots of valuable analyses of mutation-based EAs (the ones that can create new individuals only by taking one of the previously sampled individuals and changing it), and it is fair to say that we understand these algorithms quite well, including the best ways to apply them~\cite{Doerr2020,AntipovD21algo,lengler2020general,DoerrJL24}, the limits of their efficiency~\cite[Section 3.6]{Doerr20complexity}, and their behavior in more complex settings such as multi-objective optimization~\cite{DoerrQu2023a,opris2025tight}.

However, EAs that create new individuals with crossover by recombining two existing individuals (these algorithms are often called \emph{genetic algorithms}, or GAs) are still not well-understood, although the first proofs of crossover being helpful have appeared more than 25 years ago~\cite{Jansen1999}. This is since the population dynamics of GAs are much more complex than in mutation-based EAs: one needs to analyze how the diversity evolves over time, which is very challenging. At the same time, in practice crossover is considered to be an essential ingredient of an efficient EA~\cite{10.1145/3009966}. 

So far, theoretical studies of GAs have shown that crossover can immensely help when the population of the algorithm is diverse: recombining two good but different solutions into a new one can often result into even better individuals~\cite{Dang2016,CorusO20}. However, the main challenge lies in obtaining a diverse population, since crossover is generally believed to reduce the diversity or have no significant effect on it (see the overview of related work in Section~\ref{sec:related-work}). For this reason, most of the theoretical studies rely in their analyses on the diversity being increased in mutation-only steps of GAs, and these studies often use either a very small probability of crossover or some mechanisms that artificially improve diversity using the problem's properties. 

In the light of the previous studies, the main result of this paper comes as a surprise: we show that crossover can actually help to increase diversity when we regularly choose the most distant individuals as parents. Our analysis reveals a very interesting interaction of mutation and crossover, where mutation helps to increase the maximal distance between individuals in the population, while crossover helps to create more pairs of individuals in that distance from each other, which prevents the algorithm from decreasing this maximal distance. Roughly speaking, if we measure diversity via this maximal distance, then crossover allows us to get a solid foothold at the current diversity level, so we stay there for long enough for mutation to increase diversity. This mechanism turns out to be almost as effective at escaping local optima as very strong diversity preserving mechanisms (see~\cite{Dang2016} for the latter). In addition, we show that crossover can support fast hill-climbing, avoiding slow-downs that might be caused by large populations.

The key to such behavior is the strategy of \emph{parent selection}. This aspect of EAs has previously been studied mostly for non-elitist algorithms, where the parent population ``dies out'' in each iteration, and only the fittest offspring survives to reproduce in the next iteration and also in multi-objective optimization (see a detailed overview in Section~\ref{sec:related-work}). However, most of those studies considered mutation-only EAs, except for the paper by~\citet{Corus2018}. The results of this paper illustrate the potential benefits of using non-uniform parent selection in elitist algorithms, especially for catalyzing the ability of crossover to help to diversify the population.

Although our analysis is performed only on one benchmark function, the intuition behind the reasons of diversification of the population is not tailored to it, which makes the insights obtained in this paper interesting for practitioners. In particular, we are optimistic that the processes similar to the ones observed in our proofs take place also on problem landscapes with \emph{local optima networks} (LONs, introduced by~\citet{OchoaTVD08}), allowing crossover driven by parent selection with a bias toward furthest individuals to effectively find such networks without any prior landscape analysis of the problem class.    

\subsection{Our contribution}
\label{sec:contribution}

The main result of this paper is Theorem~\ref{thm:k-general}, which shows that the \muoga (formally defined in Section~\ref{sec:prelim}) which with probability $p$ performs crossover and chooses the furthest parents for it can find the global optimum of \jump (with jump size $k$) in expected $O(4^k n / p)$ fitness evaluations, if its population size $\mu$ is at least $c \cdot 4^k \log(n) / p$ for some constant $c$ and if its population is already on the \plateau of \jump. The main reason that this parent selection helps us is not because of the last step when we can choose the furthest parents to create the optimum (as it might seem in the context of the previous studies of the \muoga) but because of its role in diversification of the algorithm's population. We also show that this parent selection allows to reach the local optima of \jump in $O(n\mu/p + n\log(n/k))$, that is, it does not slow the algorithm down on ``easy'' parts of the problem and is even beneficial for hill-climbing tasks.   
 
The main technical novelty that this paper brings to the toolbox of runtime analysis of EAs is the new diversity measure that we study. This measure captures the maximum distance between the individuals in the population and the number of disjoint pairs of individuals in that distance (see Definition~\ref{def:most-distance-and-number}), and our analysis shows that it is a meaningful description of the population structure. Some of our results for this measure (Lemma~\ref{lem:basic-properties} and Lemma~\ref{lem:structural}) are proven without any bind to the fitness function or parent selection method, which, we believe, creates a decent foundation for more general analyses of the \muoga (even with uniform parent selection). 

\subsection{Related Work}
\label{sec:related-work}

\textbf{Theoretical studies of crossover.} Numerous theoretical studies are dedicated to finding the best practices of using crossover. Many of them analyze evolutionary algorithms with and without crossover on the \jump benchmark.\footnote{\jump is formally defined later in Section \emph{Preliminaries}.} This pseudo-Boolean test function has a unique global optimum, and all points in radius $k$ (which is a parameter of the benchmark) from the global optimum are local optima with the second-best function value. Typically, evolutionary algorithms easily reach these local optima and then they either have to perform a ``jump'' by mutating one local optimum precisely into the global optimum or by correctly recombining two local optima via crossover. This benchmark was first introduced by~\citet{Jansen1999,Jansen2002}, and they showed that classic mutation-only evolutionary algorithms need $\Theta(n^k)$ time to find the optimum, while crossover-based \muoga can do that in $O(4^k \cdot \mathrm{poly}(n))$ expected time (where $n$ is the dimensionality of the problem, usually called the \emph{problem size}). The gap between these two performances is super-polynomial, when $k = \omega(1) \cap O(\log(n))$. The main disadvantage of this result is that it was proven only for an algorithm with an extremely small probability to perform crossover (much smaller than commonly used in practice). In particular, the analysis was built on the idea that the algorithm first diversifies its population through mutation, covering many different local optima, and then crossover can recombine two individuals into the optimal solution. 

Many follow-up works that analyzed the \muoga on \jump can be divided into two categories. \emph{The first group} are the works focusing on the analysis of the vanilla \muoga with reasonable crossover rate. The central result in this line of work was obtained by~\citet{Dang2017}, who showed that, due to mutation, a small amount of diversity can spread through the population, so that the majority genotype constitutes only a $1 - \Omega(1)$ fraction of the population. Recently, this result was refined by~\citet{Doerr2024} who proved that the diversity created in this manner is maintained for an exponential time. However, in both cases, the \emph{amplitude} of the diversity (that is, the maximum distance between individuals) can be small, so the upper bound on the runtime proven in those papers is $\tilde O(n^{k - 1})$ fitness evaluations, that is, by factor $\tilde \Theta(n)$ better than the runtime $\Theta(n^{k})$ of classic mutation-based EAs. 

\emph{The second group} consists of works that analyzed a modified version of the algorithm, which introduced various mechanisms to speed-up diversification of the population and/or preserve the diversity of the population for a longer time~\cite{Koetzing2011a,Opris2024,Dang2017,Dang2016,Ren2024}. Some of the algorithms proposed in these works even achieved $O(4^{k} + n\log(n))$ runtime on \jump, which is the best theoretically possible performance for a large class of evolutionary algorithms, as it was shown by~\citet{Opris2025Jump}. Unfortunately, the arguments used in the proofs of these works were not applicable to the vanilla \muoga, and they do not explain its population dynamics. On the positive side, these works introduced various diversity measures which help to capture the state of the population, which is especially helpful for empirical studies. 

Interestingly, many analyses in both of these groups of works were similar in having a pessimistic view on crossover iterations: their arguments used mutation-only iterations as the main source of diversity improvements, while crossover iterations were assumed to produce the worst possible individuals. The purpose of using low crossover probability or diversity-preserving tie-breaking rules was primarily in dealing with the negative effect of such presumably bad crossover offspring. This resulted in the common belief that crossover does not help to diversify population or even prevent the algorithm from doing it.

There are many theoretical results studying the \muoga on other benchmark functions, such as OneMax~\cite{CorusO20}, RoyalRoad~\cite{Jansen2002,Jansen2005} (including a variant of this function for permutation spaces~\cite{Opris2025Royal}) or LeadingOnes~\cite{CerfL24}, and all of them to different extent show benefits of using crossover in evolutionary algorithms. However, the main argument there is the same as on \jump: the algorithm first obtains some diversity via mutations and then exploits it with crossover to get better progress than it can get by only using mutation. Crossover has also been analyzed on multiple combinatorial problems such as graph coloring~\cite{Fischer2005,Sudholt2005} and all-pairs shortest path~\cite{DoerrT09,Doerr2012}, and although the arguments used in their proofs are very problem-specific, they nevertheless demonstrate various ways in which crossover can be useful. Recently, following a rapid emergence of theoretical analyses of multi-objective optimizers, several works proved the positive impact of crossover also in this domain~\cite{OPRIS2026,Dang2024,DoerrQu2022}.

\textbf{Parent selection.} Strategy of selecting individuals to reproduce is one of the crucial parts of most evolutionary algorithms that allows to balance between exploration (finding new promising areas of the search space) and exploitation (thoroughly searching in an area of high interest). Various parent selection mechanisms were studied theoretically in the area of multi-objective optimization~\cite{ecj04,tcs20,Bian2022PPSN} and in the area of non-elitist EAs~\cite{Baker1989,Goldberg1990,Motoki02,Lehre2011,DangELQ22,Corus2018}. Most these studies focus on mutation-only EAs, with exception being~\cite{Corus2018}. In this paper, however, the proofs pessimistically assume that crossover offspring are always the worst possible solutions (with exception of their analysis of GAs on some trivial problems), which lines up with the analyses of the \muoga mentioned earlier (and therefore, with understanding of the role of crossover).
For the elitist single-objective EAs (such as the \muoga studied in this work) parent selection was only studied for the mutation-only counterpart of the \muoga, the \muoea, by~\citet{CorusLOW21} (further developed by~\citet{CorusOZ25}), where they showed a counter-intuitive result that giving the priority to the worst individuals to reproduce allows avoiding getting stuck in local optima and to find many optima in one run.

\textbf{Summary.} From this overview of existing results one can see that it is well understood that crossover can be beneficial when recombining individuals of a diverse population, but its role in diversifying a population is still unclear. This implies a lack of understanding of population dynamics of crossover-based algorithms. It is also notable that it is not known, which parent selection strategies work best for crossover-based EAs. In this work, we aim at making progress toward covering these two blind spots of theory of evolutionary computation. 

\section{Preliminaries} \label{sec:prelim}
\subsection{Notation}

We denote by $|A|$ the cardinality of a finite set $A$, by $\log$ the logarithm to base $2$ and by $\ln$ the logarithm to base $e$. For $n \in \N$ we write $[n] \coloneqq \{1,2,\ldots,n\}$. For $x \in \{0,1\}^n$ and $I \subset [n]$ we denote by $\ones{x}^I \coloneqq \sum_{i \in I} x_i$ the number of ones and by $\zeros{x}^I \coloneqq \sum_{i \in I} (1-x_i)$ the number of zeros at positions from $I$, respectively. For $I \subset [n]$ and $x,y \in \{0,1\}^n$, let $H_I(x,y) \coloneqq \sum_{i \in  I}|x_i-y_i|$ be the number of positions from $I$ in which $x$ and $y$ differ. If $I=[n]$ we also write $\ones{x}=\ones{x}^I$, $\zeros{x}=\zeros{x}^I$ and $H(x,y)=H_I(x,y)$, the latter denoting the \emph{Hamming distance} between $x$ and $y$. We call $x$ \emph{complementary to $y$ with respect to $I$} if $x_i \neq y_i$ for all $i \in I$. 
The \jump benchmark is defined as
\[	
\jump(x) \coloneqq
\begin{cases}
	\ones{x}+k & \text{if } \ones{x} = n \text{ or }  \ones{x} \leq n-k, \\
	n - \ones{x} & \text{otherwise}.
\end{cases}
\]
By \plateau we mean the set of all search points with $n-k$ ones in the context of $\jump(x)$. 

\subsection{Algorithm} 

We define a slightly modified version of the classic \muoga incorporating parent selection mechanisms. We use the following notation. For $t \in \mathbb{N}_0$ the population of the \muoga in the beginning of iteration $t$ is  denoted by $P_t$, and it a multiset of $\mu$ individuals, that is, $\{x_1, \ldots , x_\mu\}$ where $x_i \in \{0,1\}^n$ for $i \in \{1,\ldots , n\}$. Let also $\mathcal{D}_m^t$ be a distribution over $P_t$ and $\mathcal{D}_c^t$ be a distribution over $P_t \times P_t$. These distributions must be chosen by the algorithm user, but in our analysis this choice is not so important, so they should be seen as some arbitrary distributions.

The \muoga that we consider in this paper starts with some initial population, and then performs iterations until some stopping criterion is met.\footnote{As in most theoretical analyses, we do not specify the stopping criterion, but we assume that the algorithm does not stop before it finds the optimum.} In each iteration with probability $p_c$ (independently from other iterations) it performs a crossover step. In this case, two parents $x_1,x_2$ are selected following distribution $\mathcal{D}_c^t$, and uniform crossover is applied to $x_1$ and $x_2$, creating offspring $y'$ that takes each bit from one of the parents (with equal probability from each parent, and independently for each bit position), and then it applies standard bit mutation to $y'$, flipping each bit independently with probability $1/n$ to generate an offspring $y$. With probability $(1-p_c)$, the algorithm performs a mutation-only iteration, selecting parent $x$ following distribution $\mathcal{D}_m^t$ and applying standard bit mutation to it to generate an offspring $y$. The offspring $y$ replaces a worst search point $z$ from the previous population if its fitness is not worse than the fitness of $z$. If there are several worst points, the algorithm selects one of them uniformly at random. The pseudocode of this algorithm is depicted in Algorithm~\ref{alg:steady-state-GA}.

\begin{algorithm2e}[t]
\SetKwIF{WithP}{ElseWithP}{Else}{with probability}{do}{else with probability}{else}{end}
  $t \gets 0$\;
  let $P_0$ contain $\mu$ search points chosen uniformly at random\;
  \While{termination criterion not met}{
    \uWithP{$p_c$}{
        choose $x_1,x_2 \sim \mathcal{D}_c^t$\;
        $y' \gets \mathrm{crossover}(x_1,x_2)$\;
        $y \gets \mathrm{mutation}(y')$\;
    }
    \Else{
        select $x \sim \mathcal{D}_m^t$\;
        $y \gets \mathrm{mutation}(x)$\;
    }
    select $z \in P_t$ uniformly at random from all search points with minimum fitness in $P_t$\;
    \lIf{$f(y) \ge f(z)$}{$P_{t+1} \gets (P_t \cup \{y\}) \setminus \{z\}$}
	$t \gets t+1$\;
    }
    \caption{\muoga with parent selection mechanisms defined by distributions $\mathcal{D}_c^t$ and $\mathcal{D}_m^t$.}
\label{alg:steady-state-GA}
\end{algorithm2e}

This definition covers many different selection mechanisms that can be used for parent selection. In particular, tournament, linear ranking, and fitness proportionate selection~\cite{Lehre2011}, power-law selection~\cite{DangELQ22} or inverse tournament selection~\cite{CorusLOW21}. The important property of these selection mechanisms that we use in our analysis is the lower bound on the probability to choose the most distant pair of individuals as crossover parents. Namely, we define
\begin{align*}
    \pfu \coloneqq \inf_{t \in \N} \Pr\left[H(x_1, x_2) = d_t \mid (x_1, x_2) \sim \mathcal{D}_c^t\right],
\end{align*}
where $d_t = \max_{y_1, y_2 \in P_t} H(y_1, y_2)$. We do not use any properties of $\mathcal{D}_c^t$, except that $\pfu = \Omega(4^r (k-r)\ln(n)/(p_c \mu))$ for a natural number $1 \leq r \leq k$ (see Theorem~\ref{thm:k-general} below). Hence, it should only depend mildly on the parameters of the Jump$_k$ problem, and also on the algorithm's population size $\mu$. We note that the uniform selection does not satisfy this condition (since it yields $\pfu = \Theta(\frac{1}{\mu^2})$), and neither do all fitness-based selection methods (since $\pfu$ can be even zero). However, selection methods that rank all \emph{pairs} of individuals by their Hamming distance (in non-ascending order) and then choose a pair according to this ranking can satisfy this condition. In particular, a power-law selection (similar to the fitness-based selection from~\cite{DangELQ22}) or an $\ell$-tournament selection with $\ell = \Omega(4^r (k-r)\ln(n)\mu / p_c)$ satisfy the condition on $\pfu$.

\subsection{Useful Tools}

\emph{Gambler's ruin problem~\cite{Feller}:} Consider a gambler who wins or loses a dollar with probabilities $p$ and $q$, respectively. He has an initial capital of $\ell$ and plays against an opponent with the initial capital $a-\ell$. The game continues until the gambler’s capital either is reduced to zero or has increased to $a$, that is, until one of the two players is ruined. The problem may also be described as a random walk starting at $0 < \ell < a$ with absorbing states at $0$ and $a$ where it moves to the right or left with probability $p$ and $q$, respectively. The process terminates when the particle reaches either $0$ or $a$. The position of the particle after $n$ steps represents the gambler’s capital after $n$ games. Then the following holds. 

\begin{lemma}
\label{lem:ruin}[Chapter~14 in \cite{Feller}]
The probability of the gambler going broke (or the particle reaches $0$) is 
$$\frac{(q/p)^a-(q/p)^\ell}{(q/p)^a-1}.$$   
\end{lemma}

\emph{Useful bounds:} The following inequalities on the amplified success rate in
$\lambda$ independent trials are useful.
\begin{lemma}[Lemma~10 in \cite{Badkobeh2015}]\label{lem:lambda-trials}
\label{lem:amplification}
For every $p\in[0,1]$ and every $\lambda \in \mathbb{N}$,
\begin{align*}
1-(1-p)^\lambda \in \left[\frac{p\lambda}{1+p\lambda},\frac{2p\lambda}{1+p\lambda}\right].
\end{align*}
\end{lemma}

\emph{Negative drift theorem:} We also need the following version of the negative drift theorem (see Theorem~3 in~\cite{KoetzDrift}).

\begin{theorem}
\label{thm:negative-drift}
Let $(X_t)_{t \geq 0}$ be random variables over $\mathbb{R}$, each with finite expectation, and let $n>0$. With $T:=\min\{t \geq 0: X_t \geq n \mid X_0 \leq 0\}$ we denote the random variable describing the earliest point that the random process exceeds $n$, given a starting value of at most $0$. Suppose there are $c$ with $0<c<n$ and $\varepsilon < 0$ such that, for all $t$, the following holds:
\begin{itemize}
    \item[(1)]  $E[X_{t+1} - X_t \mid X_0,...X_t,T > t] \leq \varepsilon$,
    \item[(2)]  $|X_t-X_{t+1}| < c$.
\end{itemize}
Then, for all $s \geq 0$, 
$$\Pr(T \leq s) \leq s \exp(-\frac{n |\varepsilon|}{2c^2}).$$
\end{theorem}

\emph{On binomial distribution:} the following lemma estimates the probability that a binomial random value exceeds its expectation.

\begin{lemma}
\label{lem:binomial}
    Let $X \sim \Bin(n, \frac{1}{2})$ with $n \ge 1$. Then for all odd $n$ we have
    \begin{align*}
        \Pr[X > \frac{n}{2}] = \frac{1}{2}.
    \end{align*}
    For all even $n$ we have
    \begin{align*}
        \Pr[X > \frac{n}{2}] \ge \frac{1}{2}\left(1 - \sqrt{\frac{2}{\pi n}}\right).
    \end{align*}
    In both cases,
    \begin{align*}
        \Pr[X > \frac{n}{2}] \ge \frac{1}{2}\left(1 - \sqrt{\frac{1}{\pi}}\right).
    \end{align*}
\end{lemma}
\begin{proof}
    Since $Y = n - X$ follows the same distribution as $X$, we can use the symmetry of binomial distribution:
    \begin{align}
        \label{eq:symmetry}
        \Pr[X > \frac{n}{2}] = \Pr[Y < \frac{n}{2}] = \Pr[X < \frac{n}{2}].
    \end{align}
    For odd $n$ we have $\Pr[X > \frac{n}{2}] + \Pr[X < \frac{n}{2}] = 1$, hence by~\eqref{eq:symmetry} we have $\Pr[X > \frac{n}{2}] = 1/2$.

    For even $n$ we need to take into account the event when $X = \frac{n}{2}$, that is, we have
    \begin{align*}
        \Pr[X > \frac{n}{2}] + \Pr[X = \frac{n}{2}] + \Pr[X < \frac{n}{2}] = 1,
    \end{align*}
    hence,
    \begin{align*}
        \Pr[X > \frac{n}{2}] = \frac{1}{2}\left(1 - \Pr[X = \frac{n}{2}]\right).
    \end{align*}
    For binomial distribution, using eq.~(1.4.18) in the book chapter by~\citet{doerr-theory-chapter} that bounds binomial coefficients, we have
    \begin{align*}
        \Pr[X = \frac{n}{2}] = \binom{n}{n/2} \frac{1}{2^n} \le \sqrt{\frac{2}{\pi n}}.
    \end{align*}
    Hence, we have
    \begin{align*}
        \Pr[X > \frac{n}{2}] \ge \frac{1}{2}\left(1 - \sqrt{\frac{2}{\pi n}}\right).
    \end{align*}
    
    Noting that the smallest even $n$ is $n = 2$ and uniting these two lower bounds, we prove the last statement of the lemma.
\end{proof}

\emph{On binomial coefficients.} The following lemma can be used to estimate some range of binomial coefficients. This result can be distilled from Lemma~8 by~\citet{DoerrW14}, where this result was formulated for the case of large $n$ (but the proof does not use it).
\begin{lemma}[Lemma~8 by~\citet{DoerrW14}]
    \label{lem:binom-lower}
    For any $n \in \N$ and any $\ell = \frac{n}{2} \pm \gamma\sqrt{n}$ we have $\binom{n}{\ell} \ge \frac{2^n}{2\sqrt{\pi n} e^{4\gamma}}$.
\end{lemma}

\section{The $(\mu+1)$ GA with Furthest Parent Selection Escapes Plateau Efficiently} \label{sec:results}

The goal of this section is to demonstrate that the $(\mu+1)$ GA with parent selection mechanisms can escape the \plateau of \jump efficiently. To achieve this, we analyze the population dynamics of this algorithm with special focus on the last part of optimization, when the whole population is on the \plateau. This analysis is structured as follows. First, to measure population diversity, we introduce $d(P_t)$ as the maximum Hamming distance between two individuals from $P_t$ and $m_t$ as the maximum number of disjoint pairs with this distance (Definition~\ref{def:most-distance-and-number}) and then give some elementary properties of this measure (Lemmas~\ref{lem:basic-properties} and~\ref{lem:structural}). We note that higher values of $d_t$ increase the probability of creating the global optimum from two individuals $x$ and $y$ with $H(x,y)=d_t$ via crossover and mutation afterwards, while higher values of $m_t$ make it less likely that $d_t$ will decrease in future iterations without first increasing. We then estimate the probability of increasing $d_t$ or $m_t$, as well as the probability of reaching the global optimum from two individuals with distance $d_t$ (Lemmas~\ref{lem:increase} and~\ref{lem:increase-2}). With all these ingredients in place, we proceed to prove the main result of this section, Theorem~\ref{thm:k-general}, which bounds the expected time to escape the \plateau. The proof strategy is as follows. As far as $d_t$ does not increase, $m_t$ performs an unfair random walk on $\{1, \ldots , \ell\}$ for a large value $\ell$ where the probability to increase $m_t$ is higher than decreasing it (Claim~\ref{cl:increase-m} in the proof of Theorem~\ref{thm:k-general}). Then, if $m_t \geq \ell$, the negative drift theorem is used to show that, with sufficiently high probability, $d_t$ increases or the global optimum is created, before $d_t$ decreases (Claim~\ref{cl:decrease-d-hard} in the proof of Theorem~\ref{thm:k-general}). These allow us to show that with constant probability $d_t$ never decreases, which allows us to condition on that and to derive an upper bound on the time the algorithm spends on the \plateau before it finds the optimum. 
 
\begin{definition}
    \label{def:most-distance-and-number}
    Let $P_t$ be a population of individuals. Let $d(P_t) \coloneqq \max\{H(x,y) \mid x,y \in P_t\}$ if $P_t \neq \emptyset$. Set $d(P_t)=-\infty$ if $P_t=\emptyset$. Further, let
    \begin{align*}
      m(P_t):&=\max\{|Y| \mid Y \subset  P_t \times P_t \text{ such that } x \neq y,\\
      &H(x,y) = d_t, \text{ and } \{x,y\} \cap \{w,z\} = \emptyset\\
      &\text{ for all } (x,y),(w,z) \in Y \}.
    \end{align*}
\end{definition}

One sees immediately the following basic properties.

\begin{lemma}
\label{lem:basic-properties}
  The following holds.
  \begin{itemize}
      \item[(1)]
      If there is only one pair $\{x,y\}$ of search points with distance $d(P_t)>0$ then $m(P_t)=1$.
      \item[(2)]
      We have that $m(P_t) \leq \lfloor{\mu/2}\rfloor$.
      \item[(3)]
      If $d_t$ increases by adding an individual $x$ to $P_t$, then $m(P_t \cup \{x\})=1$.
  \end{itemize}
\end{lemma}

\begin{proof}
The only multiset of most distant pairs of search points is $Y=\{\{x,y\}\}$ which proves (1) and the number of disjoint pairs in $P_t$ is at most $\lfloor{\mu/2}\rfloor$ which proves (2). Every pair of individuals in $P_t \cup \{x\}$ with maximum distance contains $x$, proving (3).  
\end{proof}

We note that $d(P_t)$ and $m(P_t)$ have not been used before in theoretical studies of the \muoga. Most of the measures that have been used for this purpose were studied by~\citet{Dang2016} who incorporated such measures into the tie-breaking mechanism of the \muoga. In particular, such measures as convex hull (the number of positions with different bits present in it), total Hamming distance and fitness sharing were maximized when breaking ties between individuals with the same fitness, and Dang et al. showed that it immensely helps the crossover to create good individuals, speeding up the \muoga on \jump. The measure closest to ours is the measure studied empirically in~\cite{Doerr2024}, where the authors tracked the maximum distance and the probability to choose a pair of parents in that distance for crossover (that is, instead of $m(P_t)$ they considered a total number of pairs of individuals in distance $d(P_t)$). Those results looked very promising, however there were no theoretical analyses of this measure.

Graphically, $m(P_t)$ can be visualized as follows. Let $G_t = (V_t, E_t)$, be a graph where each vertex of $V_t$ represents one of the $\mu$ individuals in $P_t$, and there exists an edge between two vertices if and only if the distance between the two corresponding individuals is $d(P_t)$. Then $m(P_t)$ is the size of a maximum matching, which is the \emph{lower bound} on the number of nodes one has to remove to have no edge left. Now we present some straightforward structural results for $d(P_t)$ and $m(P_t)$ of any population $P_t$, independent of the underlying fitness function.

\begin{restatable}{lemma}{structuralresults}
\label{lem:structural}
    Let $P_t$ be a population with $d_t \coloneqq d(P_t)>0$, and $m_t \coloneqq m(P_t)$. Then the following properties hold.
    \begin{enumerate}
        \item Suppose $m_t=1$.
        Then there are at most two $y \in P_t$ such that $d(P_t \setminus \{y\}) < d_t$. For the other $x \in P_t$, $d(P_t \setminus \{x\}) = d_t$.
        \item Suppose $m_t>1$. Then there is $S \subset P_t$ with $|S|=2m_t$ such that $d(S) = d_t$ and $m(S)=m_t$. Hence, $m(P_t \setminus \{y\}) = m_t$ for all $y \in P_t \setminus S$.
    \end{enumerate}
\end{restatable}

\begin{proof}
  (1): Let $x,y \in P_t$ with $H(x,y)=d_t$. Removing any other individual cannot increase $d_t$ (since we reduce the number of pairs) and it does not reduce $d_t$, since $x$ and $y$ are still in the population. Hence, $d_t$ does not change.
  (2): Let $Y \subset P_t \times P_t$ with $|Y|=m_t$ such that $H(x,y)=d_t$ and $\{x,y\} \cap \{w,z\} = \emptyset$ for all $(x,y),(w,z) \in Y$ (which exists by the  definition of $m_t$). Let  
  $$S \coloneqq \{x \in P_t \mid \exists y \in P_t \text{ such that } (x,y) \in Y \vee (y,x) \in Y\}.$$
  Note that $S$ contains exactly the elements appearing in a pair in $Y$ and hence, $|S|=2m_t$. Removing any element not from $S$ cannot increase $m_t$ (since then $m_t$ is the maximum over a smaller set) and it also does not change $Y$, hence we still have $m_t$ disjoint pairs in distance $d_t$.
\end{proof}

In the remainder of this section we estimate the expected runtime of the \muoga on \jump when it starts with the whole population already on the \plateau, which is shown in Theorem~\ref{thm:k-general}.
Note that when the whole population is on the \plateau, then $d_t$ can only take even values (since the distance between two individuals on the \plateau can be only even). The following lemma provides lower bounds on the probabilities that, in a single iteration of Algorithm~\ref{alg:steady-state-GA}, the value $m(P_t)$ or $d(P_t)$ increases.

\begin{restatable}{lemma}{valuesincreaseone}
\label{lem:increase}
    Consider an iteration $t$ of the \muoga with furthest parent selection where $P_t$ is on the \plateau of \jump, $d_t:=d(P_t)$ and $m_t:=m(P_t)$. Then with probability at least 
        $$\frac{p_c \cdot \pfu}{4} \cdot \left(\frac{1}{2}\right)^{d_t} \left(1-\frac{2m_t+1}{\mu}\right)\eqqcolon p_{\text{up}}(d_t,m_t)$$
    either $d_t$ increases or $m_t$ increases.
\end{restatable}

\begin{proof}
    If $m_t = \lfloor{\mu/2}\rfloor$ then the lemma trivially holds since $p_{\text{up}}(d_t,m_t) \leq 0$. So suppose that $m_t < \lfloor{\mu/2}\rfloor$. We show that the probability $p_{\text{up}}(d_t,m_t)$ results from crossover on a most distant pair, producing a specific outcome $x$ on \plateau to increase the number of disjoint pairs by one, omitting mutation and not removing any individual from such a pair afterwards. With probability $p_c \cdot \pfu$, the \muoga applies crossover on a pair of most distant parents. Suppose that this happens, and let $w,z \in P_t$ be individuals with $H(w,z)=d_t$, on which crossover is applied. Let $b \coloneqq d_t/2$ and $I \subset [n]$ be the set of positions $i$ with $w_i \neq z_i$. We call any $i \in I$ a \emph{complementary position}. Since $w,z$ are on \plateau, we have $\zeros{w}=\zeros{z}=k$ and hence, $\ones{w}^I=\ones{z}^I=\zeros{w}^I=\zeros{z}^I=b$ (see Figure~\ref{fig:lem:increase}). Let $J \coloneqq [n] \setminus I$ be the \emph{non-complementary} part. By Lemma~\ref{lem:structural}(2) there is $S \subset P_t$ such that $|S|=2m_t < \mu$ and $m(S)=m_t$. Let $y$ be any individual from $P_t \setminus S$, and let $a \coloneqq |y|_0^I$ (see Figure~\ref{fig:lem:increase}). We now aim at generating a crossover offspring $x$ such that (i) it is on the \plateau and (ii) it is in distance at least $d_t$ from $y$. If we succeed in it, and then this offspring is not changed by mutation, then adding $x$ into the population (without yet removing any individual) increases either $m_t$ (if $H(x, y) = d_t$, then they form a new disjoint pair which can be added to $S$) or even $d_t$ (if $H(x, y) > d_t$).
   
    To show how we can generate such an $x$, we first note that $x$ is on the \plateau if (and only if) it has exactly $b$ zero-bits (and thus $b$ one-bits) in $I$, since in $J$ it coincides with both parents that have $k - b$ zero-bits there (see Figure~\ref{fig:lem:increase}). The fact that $x$ coincides with its parents in $J$ also implies that $H_J(x, y)$ is the same as $H_J(w, y)$, which is at least the difference in the number of zero-bits in $J$ in these two individuals, that is, $H_J(x, y) \ge |(k - a) - (k - b)| = |a - b|$. Now we want to set bits of $x$ in $I$ in such way that there were at least $d_t - |a - b| = 2b - |a - b|$ different positions. 
    
    Consider first the case when $a \ge b$ (that is, majority of the bits of $y$ in $I$ are zeros). Then we can put $b$ one-bits of $x$ in the positions that are zeros in $y$, and put zeros in the rest of them. It implies that $y$ and $x$ are different in $b$ positions that are ones in $x$, and in all $2b-a$ positions that are ones in $y$, that is, $H_I(x, y) = b + 2b - a = 2b - |a - b|$. If $a<b$ then we can put $b$ zero-bits of $x$ in the positions that are ones in $y$, and put ones in the rest of them. It implies that $y$ and $x$ are also different in $b$ positions that are zeros in $x$, and in all $2b-a$ positions that are zeros in $y$, implying $H_I(x, y) = b + 2b - a = 2b - |a - b|$. Hence, we always can create offspring $x$ that is in distance at least $d_t$ from $y$ by choosing correct values for the bits in $I$, probability of which is at least $(\frac{1}{2})^{d_t}$. This offspring is then not changed by mutation with probability at least $(1 - \frac{1}{n})^n \ge \frac{1}{4}$, where we used $n \ge 2$ (since it is the minimum $n$ for which \jump can be defined).

    \begin{figure}
        \begin{center}
            \begin{tikzpicture}
                \node [left] at (-0.3, 3.25) {$w:$};
                \draw (-0.1, 3) rectangle (6.1, 3.5);
                \foreach \i in {0,...,10} {
                    \node at (0.2 * \i + 0.1, 3.25) {$1$};
                };
                \foreach \i in {11,...,23} {
                    \node at (0.2 * \i + 0.1, 3.25) {$0$};
                };
                \foreach \i in {24,...,29} {
                    \node at (0.2 * \i + 0.1, 3.25) {$1$};
                };

                \node [left] at (-0.3, 2.25) {$z:$};
                \draw (-0.1, 2) rectangle (6.1, 2.5);
                \foreach \i in {0,...,10} {
                    \node at (0.2 * \i + 0.1, 2.25) {$1$};
                };
                \foreach \i in {11,...,17} {
                    \node at (0.2 * \i + 0.1, 2.25) {$0$};
                };
                \foreach \i in {18,...,23} {
                    \node at (0.2 * \i + 0.1, 2.25) {$1$};
                };
                \foreach \i in {24,...,29} {
                    \node at (0.2 * \i + 0.1, 2.25) {$0$};
                };
                
                \draw [dashed] (3.6, 4) -- (3.6, -0.2);
                \draw [dashed] (-0.1, 4) -- (-0.1, 3.5);
                \draw [dashed] (6.1, 4) -- (6.1, 3.5);
                \draw [dotted] (2.2, 3.8) -- (2.2, 3.5);
                \draw [dotted] (4.8, 3.8) -- (4.8, 3.5);

                \draw (-0.1, 4) to[bend left=6pt] node [pos=0.5, above] {$J$} (3.6, 4);
                \draw (3.6, 4) to[bend left=6pt] node [pos=0.5, above] {$I$}  (6.1, 4); 

                \draw [<->] (2.2, 3.6) -- node [pos=0.5, above] {$k - b$} (3.6, 3.6);
                \draw [<->] (3.6, 3.6) -- node [pos=0.5, above] {$b$} (4.8, 3.6);
                \draw [<->] (4.8, 3.6) -- node [pos=0.5, above] {$b$} (6.1, 3.6);

                \node [left] at (-0.3, 1.25) {$y:$};
                \draw (-0.1, 1) rectangle (6.1, 1.5);
                \node at (1.75, 1.25) {$k - a$ zeros};
                \node at (4.85, 1.25) {$a$ zeros};

                \node [left] at (-0.3, 0.25) {$x:$};
                \draw (-0.1, 0) rectangle (6.1, 0.5);
                \foreach \i in {0,...,10} {
                    \node at (0.2 * \i + 0.1, 0.25) {$1$};
                };
                \foreach \i in {11,...,17} {
                    \node at (0.2 * \i + 0.1, 0.25) {$0$};
                };
                \node at (4.85, 0.25) {$b$ zeros};

            \end{tikzpicture}
            \caption{Illustration to the proof of Lemma~\ref{lem:increase}, where the two parent individuals are $w$ and $z$, another individual to which we aim to create a pair is $y$ and the crossover offspring is $x$.}
            \label{fig:lem:increase}
        \end{center}
    \end{figure}
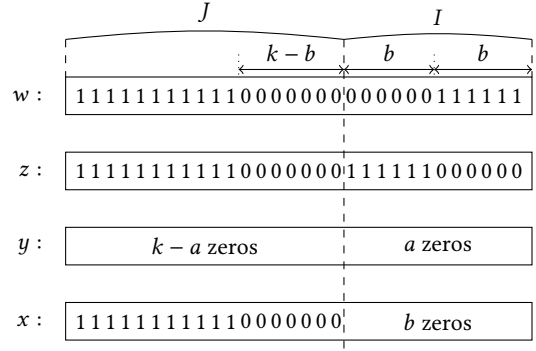
   
    Finally, we also take into account the selection step, when we remove a random individual from the population. If $H(x,y)>d_t$, the lemma holds since $m(P_t) \geq 1$ and $d_t$ increases if $y$ is not removed (happening with probability $1-1/\mu$). If $H(x,y)=d_t$, then $m(P_{t+1}) = m(S \cup \{x,y\})=m_t+1$ and $m_t$ increases if no individual from $S \cup \{y\}$ is removed, which happens with probability $1-(2m_t+1)/\mu$. Bringing together the probabilities to perform crossover on the most distant parents, creating the right offspring, not mutating it and then not removing a wrong individuals in the selection completes the proof.
\end{proof}

We also provide lower bounds on the probabilities of increasing $d_t$ and of creating the global optimum, each independent of $m_t$.

\begin{restatable}{lemma}{valuesincreasetwo}
\label{lem:increase-2}
    Consider an iteration $t$ of the \muoga with furthest parent selection where $P_t$ is on \plateau, $d_t:=d(P_t)$ and $m_t:=m(P_t)$. Suppose that $k \leq n/3$. Then the following holds.
    \begin{itemize}
        \item[(1)] If $d_t<2k$, then with probability at least 
        $$p_c \cdot \pfu \cdot \left(\frac{1}{2}\right)^{d_t} \frac{1}{3en} \left(1-\frac{1}{\mu}\right) \eqqcolon p_{\text{up}}^*(d_t)$$
        $d_t$ increases.
        \item[(2)]
        With probability at least  
        $$p_c \cdot \pfu \cdot \left(\frac{1}{2}\right)^{d_t} \frac{1}{en^{k-d_t/2}} \eqqcolon p_{\text{opt}}(d_t)$$
        the global optimum is created.
    \end{itemize}
\end{restatable}

\begin{proof}
    
    Statement (1): The following scenario increases $d_t$. The \muoga performs crossover, chooses two most distant parents $w,z$ and creates $x$ on \plateau via crossover such that $H(w,x) = d_t$ (e.g. $x=z$). Afterwards, it flips exactly a one bit and a zero bit in $x$ at positions where $w$ and $z$ coincide and remains the other bits unchanged to create $x'$. Then $H(w,x') = d_t+2$. Finally, $x'$ should not be removed from $P_t$. Hence, $d_t$ is increased with probability at least
    \begin{align*}
    p_c \cdot \pfu &\cdot \left(\frac{1}{2}\right)^{d_t} \frac{(k-d_t/2)(n-k-d_t/2)}{n^2} \cdot\left(1-\frac{1}{n}\right)^{n-2} \left(1-\frac{1}{\mu}\right) \\
    &\geq p_c \cdot \pfu \cdot \left(\frac{1}{2}\right)^{d_t} \frac{1}{3en} \left(1-\frac{1}{\mu}\right).
    \end{align*}

    Statement (2): To create the optimum, the \muoga can perform crossover, choose two most distant parents $w,z$ and create the outcome $x$ with $x_i=1$ at all positions $i$ where $w_i \neq z_i$. Then it can flip the remaining $k-d_t/2$ zeros with mutation and remains the other bits unchanged. This happens with probability at least
    \begin{align*}
    &p_c \cdot \pfu \cdot \left(\frac{1}{2}\right)^{d_t} \frac{1}{n^{k-d_t/2}} \left(1-\frac{1}{n}\right)^{n-k+d_t/2}
    \end{align*}
    concluding the proof.
\end{proof}

Finally we are able to prove the main theorem in this paper which gives an upper bound for the expected time until the \muoga with furthest parent selection finds the global optimum, starting from a population on the \plateau. The theorem holds for arbitrary values of $2 \leq k \leq n/3$ but requires that $\mu \ge c \cdot 4^r (k-r)\ln(n)/(p_c \pfu)$ for a sufficiently large constant $c$, where $p_c$ and $\pfu$ may depend on $\mu$.

\begin{theorem}\label{thm:k-general}
  Consider the \muoga with furthest parent selection on \jump for $2 \leq k \leq n/3$. Let $r \in [k-1]$. Suppose $\mu \ge c \cdot 4^r (k-r)\ln(n)/(p_c \pfu)$ for a sufficiently large constant $c$ and suppose that it starts with an arbitrary population $P_0$, all individuals of which are on the \plateau. Then the expected number of fitness evaluations that the algorithm takes to find the global optimum is $O(4^r \cdot n^{k-r}/(p_c \pfu))$.
\end{theorem}

\begin{proof}
Let $d_t \coloneqq d(P_t)$ and $m_t \coloneqq m(P_t)$ from Definition~\ref{def:most-distance-and-number}. The main goal is to show that with at least constant probability, the value $d_t \in \{0,2, \ldots , 2r-2\}$ never decreases before the all-ones string is created. To this end, we first show in Claim~\ref{cl:increase-m} that for every $d_t \in \{0,2, \ldots , 2r-2\}$, with sufficiently high probability, either a pair of individuals with distance at least $d_t+2$ is created, or $m_t$ increases to a large value without decreasing $d_t$. Then, in Claim~\ref{cl:decrease-d-hard} we apply the negative drift theorem to prove that, with sufficiently high probability, a decrease of $d_t$ does not occur before either its increase (if $d_t < 2r$ so that it is possible) or the finding of the global optimum. Afterwards we conduct the desired runtime analysis with the method of typical runs~\citep[see][Section 11]{Wegener2002}. For convenience, we define $\ell \coloneqq \ell(\delta,k,n) \coloneqq \lfloor{20\delta \ln(kn)}\rfloor$ for $1 \leq \delta \leq k-r$, where $\delta$ is an auxiliary variable introduced to apply the negative drift theorem for different values of $d_t$ in Claim~\ref{cl:decrease-d-hard}.
Let $c$ from the condition on $\mu$ be at least 640. Then due to $\ln(nk) \leq 2 \ln(n)$
\begin{align}
\label{eq:ell-und-mu}
\frac{\ell}{\mu} \leq \frac{20\delta \ln(kn)}{320 \cdot 4^r(k-r)\ln(kn)/(p_c \pfu)} \leq \frac{p_c \pfu}{16 \cdot 4^r}. 
\end{align}

\begin{restatable}{claim}{claimone}
\label{cl:increase-m}
    Let $1\leq \delta \leq k-r$. Then for every $d_t \in \{2, \ldots , 2r\}$ the probability to increase $d_t$ (if $d_t < 2r$) or to increase $m_t$ to at least $\ell$ in a future iteration before decreasing $d_t$ is at least $1-\frac{5}{8}\frac{1}{2^{2r-d_t}}$.
\end{restatable}
\begin{proofofclaim}
  Without loss of generality we assume that $m_t<\ell \leq \lfloor{\mu/2}\rfloor$. We first observe that in order to decrease $d_t$, at least $m_t$ different individuals must be removed from the population. Therefore, it is only possible to achieve this in a single iteration when $m_t = 1$. Let $\mathcal{E}^-$ be the event that in one iteration the algorithm either decreases $m_t$, if $m_t > 1$, or it decreases $d_t$, if $m_t = 1$. Since we assume that all individuals are on the \plateau, the removed individual is selected uniformly at random, hence the probability that $\mathcal{E}^-$ occurs is at most $2m_t/\mu$ (by Lemma~\ref{lem:structural}(1) and~(2)). Let $\mathcal{E}^+$ be the event to increase $m_t$ or to increase $d_t$ in one iteration. By Lemma~\ref{lem:increase}, for each $m_t \in \{1, \ldots , \ell-1\}$ and $d_t \in \{0,2, \ldots , 2r\}$, the probability that $\mathcal{E}^+$ occurs is at least
  \begin{align}
  \label{eq:c(d)}
    & \frac{p_c \cdot \pfu}{4} \cdot \left(\frac{1}{2}\right)^{d_t}  \cdot \frac{\mu-2\ell}{\mu} \geq \frac{p_c \cdot \pfu}{5} \cdot \left(\frac{1}{2}\right)^{d_t} \eqqcolon c(d_t)
  \end{align}
  owing to $2\ell/\mu \leq  \frac{p_c \pfu}{8} \cdot \left(\frac{1}{2}\right)^{2r} \leq 1/5$ (by Eq.~\eqref{eq:ell-und-mu}). 
  
  Now we describe the process of increasing and decreasing $d_t$, as well as evolution of the value $m_t$ if $d_t$ remains unchanged over the time $t$, according to the following Markov chain with states $S_0, S_1,\ldots, S_{\lfloor{\mu/2}\rfloor}$ and $T$. State $T$ will be entered if $d_t$ increases, and state $S_0$ if $d_t$ decreases. Furthermore, $S_i$ for $i \geq 1$ will be entered if $i=m_t$ and $d_t$ does not change. The transition probability from $S_i$ to $S_{i-1}$ is at most $2i/\mu$ and from $S_i$ to $\{S_{i+1},T\}$ is at least $c(d_t)$. The transition probability from $S_i$ to $S_j$ for $|j-i| \geq 2$ is zero. Here, $T$ and $S_0$ are absorbing states. If the current state is $S_i$ for $i \geq 2$, then $d_t$ cannot decrease (since $d_t$ can decrease only in $S_1$). Hence, the aim is to determine the probability to reach $T$ or $S_\ell$ before $S_0$. For $1 \leq i \leq \ell-1$ the probability of reaching $S_{i+1}$ or $T$ before $S_{i-1}$ starting from $S_i$ is given by 
  \begin{align*}
    \Pr(B_i \mid A_i \cup B_i) &= \frac{\Pr(B_i)}{\Pr(A_i \cup B_i)} \geq \frac{\Pr(B_i)}{\Pr(B_i)+ \Pr(A_i)}\\
    &\geq \frac{\Pr(B_i)}{\Pr(B_i)+ 2i/\mu} \geq \frac{c(d_t)}{c(d_t)+2i/\mu} \\
    &\geq \frac{c(d_t)}{c(d_t)+2(\ell-1)/\mu} \eqqcolon p_{\text{u}},
  \end{align*}
  where $A_i$ denotes the event that we reach $S_{i-1}$, and $B_i$ the event to reach $S_{i+1}$ or $T$. Hence, the probability of reaching $S_0$ before $T$ or $S_{\ell}$ can be bounded from above by the probability that an unfair random walk with $\ell+1$ states $Z_0, \ldots , Z_\ell$ with transition probability $p_{\text{u}}$ from $Z_i$ to $Z_{i+1}$ and $p_{\text{d}}=1-p_{\text{u}}$ from $Z_{i+1}$ to $Z_i$ for all $i \in \{0\} \cup [\ell-1]$ reaches $Z_0$ before $Z_\ell$ when starting from $Z_1$. This stochastic process is also known as the \emph{Gambler's Ruin Problem} (Lemma~\ref{lem:ruin}) this probability is

  \begin{align*}
    \frac{(p_{\text{d}}/p_{\text{u}})^{\ell}-(p_{\text{d}}/p_{\text{u}})}{(p_{\text{d}}/p_{\text{u}})^{\ell}-1} = \frac{p_{\text{d}}}{p_{\text{u}}} \cdot \frac{1-(p_{\text{d}}/p_{\text{u}})^{\ell-1}}{1-(p_{\text{d}}/p_{\text{u}})^{\ell}} \leq \frac{p_{\text{d}}}{p_{\text{u}}}.
  \end{align*}
  This concludes the proof of this claim since 
  \begin{align*}
    \frac{p_{\text{d}}}{p_{\text{u}}}&= \frac{2(\ell-1)}{c(d_t) \mu} \leq \frac{p_c \pfu}{8 \cdot 4^{r}} \cdot \frac{1}{\frac{p_c \pfu}{5 \cdot 2^{d_t}}} = \frac{5}{8}\frac{1}{2^{2r-d_t}}.
  \end{align*}
\end{proofofclaim}

The upper bound on the probability that a decrease of $d_t$ occurs before an increase (if $d_t < 2r$) or the creation of the global optimum is as follows.

\begin{restatable}{claim}{claimtwo}
\label{cl:decrease-d-hard}
    Let $1\leq \delta \leq k-r$. Consider an iteration $t$ of the \muoga and suppose that $m_t \geq \ell$. Then the following holds.
    \begin{enumerate}
        \item[(1)] If $0 < d_t < 2r$, the probability to decrease $d_t$ before increasing it or creating the global optimum is at most $19/(10k)$.
        \item[(2)] If $d_t = 2r$ and $\delta = k - r$, the probability to decrease $d_t$ before creating the global optimum is at most $3/4$.
    \end{enumerate}
\end{restatable}

\begin{proofofclaim}
  Let $c(d_t)$ and $p_u$ be as in the proof of Claim~\ref{cl:increase-m}. Denote by $A_t$ the event that $m_t$ or $d_t$ decreases, by $B_t$ that $m_t$ increases (and hence not $d_t$), but the optimum is not created, and by $C_t$ the event that $d_t$ increases or the optimum is created. Further, denote by $\tilde{B}_t$ the event from (the proof of) Lemma~\ref{lem:increase} and by $\tilde{C}_t$ the event from (the proof of) 
  Lemma~\ref{lem:increase-2}(1) if $d_t < 2r$, or Lemma~\ref{lem:increase-2}(2) if $d_t = 2r$. Note that $\tilde{C}_t \subset C_t$ and $\tilde{B}_t \subset B_t \cup C_t$. Further, $\tilde{B}_t$ and $\tilde{C}_t$ are disjoint since in the latter, at least one bit is flipped, whereas in the former, no bits are flipped during mutation. Therefore, we can pessimistically assume that $B_t=\tilde{B}_t$ and $C_t=\tilde{C}_t$ as this only reduces the probability for increasing $d_t$ if $d_t < 2r$ or creating the optimum if $d_t = 2r$ before decreasing $d_t$ and $m_t$ can only change by at most one if $d_t$ does not decrease. For the same reason, we can assume that $\Pr(A_t) = 2m_t/\mu$, $\Pr(B_t) = c(d_t)$ (see Equation~\ref{eq:c(d)}), and $\Pr(C_t) = p_{\text{up}}^*(d_t)$ if $d_t<2r$, while $\Pr(C_t) = p_{\text{opt}}(d_t)$ if $d_t=2r$.

  We can also assume that in each iteration of the \muoga, one of the events $A_t,B_t$ or $C_t$ occurs since otherwise $d_t$ and $m_t$ do not change and we make no progress. 

  At first we derive an upper bound for the probability of decreasing $d_t$ within $s \in \mathbb{N}$ iterations, assuming that $d_t$ does not increase for each $d_t \in \{0,2, \ldots , 2r-2\}$ and the optimum is not created during this time. To this end, we apply the negative drift theorem (see Theorem~\ref{thm:negative-drift}) on $X_t=\ell - m_t$. The target $X_{t+1}=\ell$ is reached when $m_t=1$ and $d_t$ decreases in iteration $t$. Note that $X_0 \leq 0$ and $X_t$ has finite expectation since $0 \leq m_t \leq \mu/2$. We define $T \coloneqq \inf\{t > 0 \mid X_t = \ell\}$ and verify both drift conditions for $X_t$.

  On the one hand, we have $\vert{X_{t+1}-X_t}\vert \leq 1$, since $m_t$ can increase or decrease by at most $1$ in each iteration, as long as the optimum is not reached. On the other hand, for all values of $X_0, \ldots ,X_t$ where the target has not been reached yet under the condition of $A_t \cup B_t$ and $X_t \geq 0$ (i.e. $m_t \leq \ell$) we have
  \begin{align*}
    E[X_{t+1}-X_t] &= \Pr(X_{t+1} = X_t + 1) \\
    &- \Pr(X_{t+1} = X_t - 1)\\
    &= \Pr(\ell - m_{t+1} = \ell - m_t + 1) \\
    &- \Pr(\ell - m_{t+1} = \ell - m_t - 1)\\
    &= \Pr(m_{t+1} = m_t - 1) - \Pr(m_{t+1} = m_t + 1)\\
    &= \frac{P(A_t)}{P(A_t) + P(B_t)} - \frac{P(B_t)}{P(A_t) + P(B_t)}\\
    &= \frac{1}{2m_t/\mu + c(d_t)} \left(\frac{2m_t}{\mu} - c(d_t)\right) \\
    &\leq \frac{1}{2\ell/\mu + c(d_t)} \left(\frac{2 \ell}{\mu} - c(d_t)\right)
    \intertext{With eq.~\eqref{eq:ell-und-mu}, we get by plugging in the right hand side for $\ell / \mu$ and the value for $c(d_t)$ from eq.~\eqref{eq:c(d)} due to $d_t < 2r$}
    &\leq \frac{1}{\frac{p_c \cdot \pfu}{8} \cdot \left(\frac{1}{2}\right)^{2r} + \frac{p_c\cdot \pfu}{5 \cdot 2^{d_t}}}\left(\frac{p_c \cdot \pfu}{8} \cdot \left(\frac{1}{2}\right)^{2r} - \frac{p_c\cdot \pfu}{5 \cdot 2^{d_t}}\right) \\
    &= \frac{1}{5 \cdot 2^{d_t} + 8 \cdot 4^r}\left(5 \cdot 2^{d_t} - 8 \cdot 4^r\right) \\
    &\leq \frac{1}{5 \cdot 2^{d_t} + 8 \cdot 2^{d_t}}\left(5 \cdot 2^{d_t} - 8 \cdot 2^{d_t}\right) \\
    &= \frac{5 - 8}{5 + 8} = -\frac{3}{13} \eqqcolon -\varepsilon < 0.
  \end{align*}
  Finally, we obtain by the negative drift theorem for all $s>0$
  \begin{align}
  \label{eq:negative-drift}
    \text{Pr}(T \leq s) &\leq s\exp\left(-\frac{\ell \varepsilon}{2}\right) = s\exp\left(-\frac{3\ell}{26}\right) \leq s\exp\left(-2\delta\ln(kn)\right).
  \end{align}
  where $\ell=\lfloor{20\delta \ln(kn)}\rfloor$ for an auxiliary $1 \leq \delta \leq k-r$. 
  Now we verify the two single claims stated in the lemma.

  (1): Recall that by our assumption
  \begin{align*}
    \Pr(C_t) = p_{\text{up}}^*(d_t) = p_c \cdot \pfu \cdot \left(\frac{1}{2}\right)^{d_t} \cdot \frac{1}{3en}.
  \end{align*}

  Then the probability that $d_t$ increases (see eq.~\eqref{eq:ell-und-mu} and~\eqref{eq:c(d)}) is at least
  \begin{align*}
    &\frac{\Pr(C_t)}{\Pr(A_t)+\Pr(B_t) + \Pr(C_t)} \geq \frac{p_{\text{up}}^*(d_t)}{2 \ell/\mu + c(d_t)+p_{\text{up}}^*(d_t)} \\
    & \geq \frac{\frac{p_c \cdot \pfu}{3en} \cdot \left(\frac{1}{2}\right)^{d_t}}{\frac{p_c \cdot \pfu}{8}\left(\frac{1}{2}\right)^{2r} + \frac{p_c \cdot \pfu}{5} \cdot \left(\frac{1}{2}\right)^{d_t}  + \frac{p_c \cdot \pfu}{3en} \cdot \left(\frac{1}{2}\right)^{d_t}} \\
    & \geq \frac{1}{\frac{3e n 2^{d_t}}{8} \cdot \left(\frac{1}{2}\right)^{2r} + \frac{3e n}{5}  + 1} \geq \frac{1}{\frac{3e n}{8} + \frac{3e n}{5} + \frac{3n}{3}} \\
    &\geq \frac{1}{4n} \eqqcolon \sigma(n) \eqqcolon \sigma,
  \end{align*}
  since by our assumption, in each iteration the event $D_t := A_t \cup B_t \cup C_t$ occurs. Denote by $\tilde{T}$ the number of iterations until $d_t$ increases. Then the probability to increase $d_t$ in $s$ iterations (which means $\tilde{T}\leq s$) if $d_t$ does not decrease (which means $T >s$) is the probability that the event $C_t$ does occur at least one time in $s$ iterations under the condition that $T>s$ holds. Hence we obtain for $s \in \mathbb{N}$
  \begin{align*}
    &\text{Pr}(\tilde{T} \leq s \mid T > s) \geq 1-\left(1-\sigma\right)^s \geq \frac{s \cdot \sigma}{1+s \cdot \sigma} \geq 1-\frac{1}{s \cdot \sigma}
  \end{align*}
  where the second inequality is due to Lemma~\ref{lem:amplification}. By letting $s=20k/(19\sigma) = 80kn/19$ we obtain $\text{Pr}(\tilde{T} \leq s \mid T > s) \geq 1-\frac{19}{20k}$. On the other hand we have by eq.~\eqref{eq:negative-drift} where we plug in $\delta=1$ and $s=80kn/19$

  \begin{align*}
    \text{Pr}(T \leq s) \leq s \exp(-2 \delta \ln(kn)) = \frac{80kn}{19k^2n^2} \leq \frac{19}{20k}
  \end{align*}
  for $n$ sufficiently large. Hence, the probability to increase $d_t$ within the next $s=80kn/19$ iterations while not decreasing $d_t$ is
  \begin{align*}
    \Pr\Big((\tilde{T} \leq s) \cap (T > s)\Big) &= \Pr(\tilde{T} \leq s \mid T > s) \cdot P(T > s) \\
    &\geq \left(1-\frac{19}{20k}\right)\left(1-\frac{19}{20k}\right) \geq 1-\frac{19}{10k},
  \end{align*}
  where we used the Weierstrass product inequality $(1-x)(1-y) \geq 1-x-y$ for $x,y \geq 0$. 

  (2): Recall that by Lemma~\ref{lem:increase-2}(2) due to $d_t=2r$
  \begin{align*}
    p_{\text{opt}}(2r) = p_c \cdot \pfu \cdot \left(\frac{1}{2}\right)^{2r} \frac{1}{en^{k-r}}
  \end{align*}
  and therefore the probability to create $1^n$ is at least (conditioned on $D_t$),
  \begin{align*}
    &\frac{\Pr(C_t)}{\Pr(A_t)+\Pr(B_t) + \Pr(C_t)} \geq \frac{p_{\text{opt}}(2r)}{2 \ell/\mu + c(2r)+p_{\text{opt}}(2r)} \\
    & \geq \frac{\frac{p_c \cdot \pfu}{en^{k-r}} \cdot \left(\frac{1}{2}\right)^{2r}}{\frac{p_c \cdot \pfu}{8} \cdot \left(\frac{1}{2}\right)^{2r} + \frac{p_c \cdot \pfu}{5} \cdot \left(\frac{1}{2}\right)^{2r}  + \frac{p_c \cdot \pfu}{en^{k-r}} \cdot \left(\frac{1}{2}\right)^{2r}} \\
    & \geq \frac{1}{\frac{e n^{k-r}}{8} + \frac{e n^{k-r}}{5} + n^{k-r}} \geq \frac{1}{2n^{k-r}} \eqqcolon \tau.
  \end{align*}
  Denote by $\hat{T}$ the time until $1^n$ is created. Then we obtain for $s \in \mathbb{N}$ iterations with a similar argument as above
  \begin{align*}
    \text{Pr}&(\hat{T} \leq s \mid T > s) \geq 1-\left(1-\tau\right)^s \geq \frac{\tau s}{1+\tau s} \geq 1-\frac{1}{\tau s}.
  \end{align*}
  For $s =4n^{k-r}$ we obtain $\text{Pr}(\hat{T} \leq s \mid T > s) \geq 1/2$. Further, by eq.~\ref{eq:negative-drift} and pluggin in $\delta = k-r$ and $s=4n^{k-r}$,
  \begin{align*}
    \text{Pr}(T \leq s) &\leq s \exp(-2\delta \ln (kn)) \leq s \exp(-2\delta \ln (n)) \\
    &= s \cdot n^{-2\delta} = 4n^{k-r}n^{-2(k-r)} \leq 1/2
  \end{align*}
  for $n$ sufficiently large due to $k>r$. With a similar argument as above, the probability of creating the optimum within the next $s=4 n^{k-r}$ iterations before decreasing $d_t$ is at least
  \begin{align*}
    \Pr\Big((\tilde{T} \leq s) &\cap (T > s)\Big) = \Pr(\tilde{T} \leq s \mid T > s) \cdot P(T > s) \\
    &\geq (1-1/2)(1-1/2) = 1/4=1-3/4,
  \end{align*}
  which concludes the proof of this claim.
\end{proofofclaim}

For a fixed but arbitrary $d_t = d$, by Claim~\ref{cl:increase-m} we either increase $m_t$ to $\ell$, increase $d_t$, or find the global optimum before $d_t$ decreases with probability at least $1-\frac{5}{8}\frac{1}{2^{2r-d}}$. Conditional on this event and pessimistically assuming that we neither increased $d_t$, nor we found the optimum (that is, we still have $d_t = d$ and $m_t \ge \ell$), we apply Claim~\ref{cl:decrease-d-hard} distinguishing two cases. If $0 < d < 2r$, then we increase $d_t$ or find the global optimum before we decrease $d_t$ with probability at least $1 - \frac{19}{10k}$. If $d = 2r$, then we find the global optimum before decreasing $d_t$ with probability $1 - \frac{3}{4} = \frac{1}{4}$. Combining the probabilities of these events for all $d \in \{2, 4, \dots, 2r\}$ and recalling that $r \le k - 1$, we conclude that the probability that $d_t$ never decreases before we find the optimum is at least
\begin{align*}
&\left(1-\frac{19}{10k}\right)^{r-1} \cdot \frac{1}{4} \cdot \prod\nolimits_{j=1}^r \left(1-\frac{5}{8}\frac{1}{2^{2r-2j}}\right)\\
&\geq \left(1-\frac{19}{10k}\right)^{k-2} \cdot \frac{1}{4} \cdot  \prod\nolimits_{j=1}^r \left(1-\frac{5}{8}\frac{1}{2^{2r-2j}}\right)\\
&\geq \left(1-\frac{19}{10k}\right)^{k-1.9} \cdot \frac{1}{4} \cdot  \left(1-\frac{5}{8} \sum\nolimits_{j=0}^\infty \frac{1}{4^j}\right)\\
&\geq \frac{1}{e^{1.9}} \cdot \frac{1}{4} \cdot  \left(1-\frac{5}{8} \frac{1}{1-1/4}\right) = \frac{1}{e^{1.9}} \cdot \frac{1}{4} \cdot  \frac{1}{6} \in \Omega(1),
\end{align*}
where we used the Weierstrass product inequality $\prod\nolimits_{j=1}^{m} (1-x_j) \geq 1-x_1 - \ldots - x_m$ for $m \in \mathbb{N}$, $x_i \in [0,1]$ and that for $0<c<k$ we have
$$\left(1-\frac{c}{k}\right)^{k-c} = \left(\left(1-\frac{c}{k}\right)^{k/c-1}\right)^c  \geq 1/e^c$$ 
owing to $(1-1/x)^{x-1} \geq 1/e$ for $x>1$. 

Denote the decrease of $d_t$ as a \emph{failure}. In the following we condition on the event that no failure occurs which happens with probability at least $\Omega(1)$ and we estimate the expected time until the optimum is found. To succeed, one may increase $d_t$ at most $r$ times if $d_t<2r$ (resulting in a value of $d_t=2r$) and then derive the all one string. By Lemma~\ref{lem:increase-2}(1), an increase of $d_t$ happens with probability at least
\begin{align*}
p_c \cdot \pfu \cdot \frac{1}{3en \cdot 2^{d_t}}\left(1-\frac{1}{\mu}\right) \geq p_c \cdot \pfu \cdot \frac{1}{6en \cdot 2^{d_t}}.
\end{align*}
Hence, the total expected waiting time (conditioned on the event that no failure occurs) until there is a pair of individuals with distance $2r$ is at most a partial sum of the geometric series
\begin{align*}
\sum_{i=0}^{r-1} \frac{6en 2^{2i}}{p_c \pfu} &\leq \frac{6en}{p_c \pfu} \sum_{i=0}^{r-1} 4^{i} = O\left(\frac{4^rn}{p_c \pfu}\right).
\end{align*}
When $d_t \ge 2r$, the optimum is created with probability at least
\begin{align*}
p_c \cdot \pfu \cdot \frac{1}{4^r} \frac{1}{en^{k-r}}.
\end{align*}
Then the expected waiting time to generate the all-ones string is at most $O(n^{k-r}4^r/(p_c \pfu))$. Therefore, the expected waiting time in total is at most 
\begin{align*}
O\left(\frac{4^rn}{p_c \pfu}+\frac{n^{k-r}4^r}{p_c \pfu}\right) = O\left(\frac{n^{k-r}4^r}{p_c \pfu}\right).
\end{align*}
If a failure occurs (i.e. $d_t$ decreases during the run), we repeat the above arguments and need another period of $O(n^{k-r}4^r/(p_c \pfu))$ iterations in expectation if no further failure occurs. Hence, the bound on the expected number of iterations follows by multiplying the above with $O(1)$, noting that in expectation, $O(1)$ repetitions are sufficient.
\end{proof}

Note that the population size in Theorem~\ref{thm:k-general} depends only logarithmically on $n$, but exponentially on $k$. If $k=O(1)$ and $p_c \pfu \in \Omega(1)$, a population size of $O(\log(n))$ suffices to escape the \plateau in expected $O(n)$ iterations. Even for $k = \omega(1)$, $p_c \pfu \in \Omega(1)$ and $\mu=c \cdot 4^k \log(n)$ we get a runtime bound of $O(4^k \cdot n)$ to create $1^n$. Moreover, it is only by a factor of $n$ worse than the bound $\Omega(4^k)$, which is the best possible lower bound achievable by a broad class of $(\mu+1)$ GAs to escape the \plateau (namely, by elitist algorithms that use any type of parent selection and perform either standard bit mutation alone or uniform crossover followed by standard bit mutation in a single iteration~\cite[Theorem~5.5]{Opris2025Jump})

\section{Furthest Parent Selection Is Also Helpful at Hill Climbing}
\label{sec:hill-climb}

In this section, we show that in addition to diversifying the population of the \muoga, furthest parent selection also helps at optimizing easy monotonic problems. For this, we estimate the runtime of the algorithm on \jump until it gets the whole population on the \plateau starting from a random population (which is the most common way to initialize the algorithm).

\begin{restatable}{lemma}{timetoplateau}
\label{lem:time-to-plateau}
     Consider the \muoga with furthest parent selection on \jump with $2 \leq k \leq n/2$ and crossover probability $0 < p_c \leq 1$. Then the expected number of fitness evaluations made by the algorithm before the whole population is on the \plateau is 
     $O\left(\frac{n \mu}{p_c\pfu} + n \log(\frac{n}{k}) \right).$
\end{restatable}

The proof of this lemma is based on the following lemma that estimates the probabilities of creating a crossover offspring that is better than the worst of its parents.
\begin{restatable}{lemma}{betteroffspring}
\label{lem:better-offspring}
    Let $f$ denote the \jump function and let $x$ and $y$ be two arbitrary different individuals such that $f(y) \le f(x)$, $f(x) < n + k$ ($x$ is not the global optimum) and $f(y) < n$ ($y$ is not a local optimum). Let $z$ be the result of uniform crossover between $x$ and $y$. There exist constants $c_1, c_2, c_3>0$ that do not depend on $x$ and $y$ such that
    \begin{enumerate}
        \item[(1)] if $f(y) \le n - 8\sqrt{k}$, then $\Pr[f(z) > f(y)] \ge c_1$;
        \item[(2)] if $f(y) > n - 8\sqrt{k}$, then $\Pr[f(z) > f(y)] \ge \frac{c_2}{\sqrt{k}}$;
        \item[(3)] $\Pr[f(z) > f(y) - \sqrt{k}] \ge c_3$, even if $f(y) = n$.
    \end{enumerate}
\end{restatable}
\begin{proof}
    We distinguish several cases depending on the fitness of $y$. First, we consider $y$ with the worst fitness values, when it is in the fitness valley between local optima and the global optimum. In the second case, $y$ is on the main slope of \jump, but not too close to the local optima, namely, in distance more than $8\sqrt{k}$ from any of them. Finally, we consider the case when $y$ is close to local optima. These cases are illustrated in Figure~\ref{fig:jump-cases}. Note that the first two cases cover statement (1) of the lemma. The last case covers statement (2). And statement (3) will be covered by all three cases together.
    
    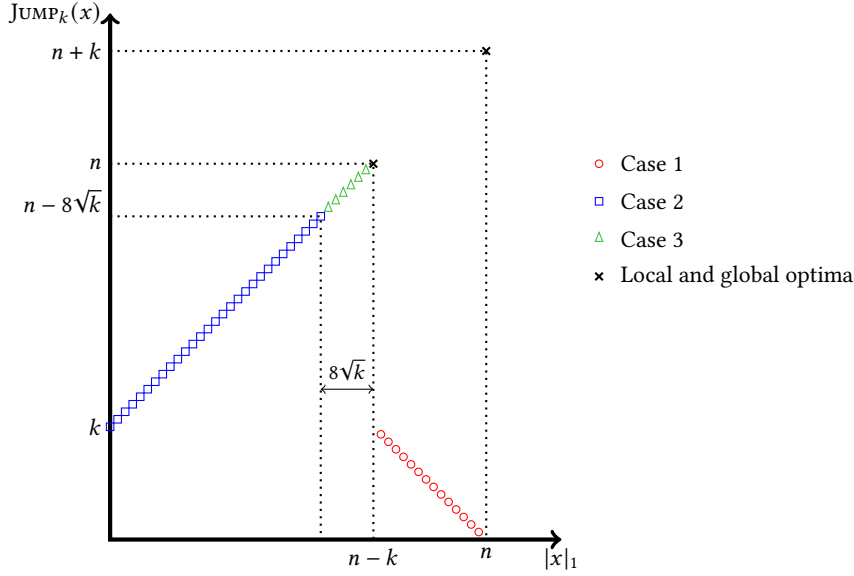
\begin{figure*}[t!]
        \begin{center}
            \begin{tikzpicture}
                \draw[<->,ultra thick] (6, 0) -- (0, 0) -- (0, 7);

                \node [below] at (6, 0) {$|x|_1$};
                \node [left] at (0, 7) {$\jump(x)$};
                \node [below] at (5, 0) {$n$};
                \node [below] at (3.5, 0) {$n - k$};
                \node [left] at (0, 1.5) {$k$};
                \node [left] at (0, 5) {$n$};
                \node [left] at (0, 4.5) {$n - 8\sqrt{k}$};
                \node [left] at (0, 6.5) {$n + k$};

                \draw [dotted, thick] (0, 5) -- (3.5, 5) -- (3.5, 0);
                \draw [dotted, thick] (0, 4.3) -- (2.8, 4.3) -- (2.8, 0);
                \draw [<->] (2.8, 2) -- node [above, pos=0.5] {\small{$8\sqrt{k}$}} (3.5, 2);
                \draw [dotted, thick] (0, 6.5) -- (5, 6.5) -- (5, 0);

                \foreach \i in {1,...,14} {
                    \node [draw, red, shape=circle, fill=none, minimum size=0.1cm, inner sep=0pt] at (5 - 0.1 * \i, 0.1 * \i){};
                };
                \foreach \i in {0,...,28} {
                    \node [draw, blue, shape=rectangle, fill=none, minimum size=0.1cm, inner sep=0pt] at (0.1 * \i, 0.1 * \i + 1.5){};
                };
                \foreach \i in {29,...,34} {
                    \node [draw, green!50!gray, isosceles triangle, rotate=90, fill=none, minimum size=0.1cm, inner sep=0pt] at (0.1 * \i, 0.1 * \i + 1.5){};
                };
                \draw [thick] (3.45, 4.95) -- (3.55, 5.05);
                \draw [thick] (3.55, 4.95) -- (3.45, 5.05);
                \draw [thick] (4.95, 6.45) -- (5.05, 6.55);
                \draw [thick] (5.05, 6.45) -- (4.95, 6.55);

                \node [draw, red, shape=circle, fill=none, minimum size=0.1cm, inner sep=0pt] at (6.5, 5){};
                \node [draw, blue, shape=rectangle, fill=none, minimum size=0.1cm, inner sep=0pt] at (6.5, 4.5){};
                \node [draw, green!50!gray, isosceles triangle, rotate=90, fill=none, minimum size=0.1cm, inner sep=0pt] at (6.5, 4){};
                \draw [thick] (6.45, 3.45) -- (6.55, 3.55);
                \draw [thick] (6.55, 3.45) -- (6.45, 3.55);

                \node [right=5pt] at (6.5, 5) {Case 1};
                \node [right=5pt] at (6.5, 4.5) {Case 2};
                \node [right=5pt] at (6.5, 4) {Case 3};
                \node [right=5pt] at (6.5, 3.5) {Local and global optima};

            \end{tikzpicture}
        \end{center}
        \caption{Illustration of the fitness values that correspond to different cases considered in the proof of Lemma~\ref{lem:better-offspring}.}
        \label{fig:jump-cases}
    \end{figure*}

    \textbf{Case 1: $y$ in the valley.} In this case, $x$ has at least as many zero-bits as $y$ (since $f(x) \ge f(y)$), and to create a strictly better offspring $z$, we need $z$ to have strictly more zero-bits than $y$. Let $D$ be the set of positions in which $x$ and $y$ differ (since $x \ne y$, $D \ne \emptyset$), and let $\Delta = |D| > 0$. The crossover offspring $z$ coincides with $y$ (and $x$) in all positions from $[n] \setminus D$, and in each position in $D$ it has either zero or one with equal probability (independently from other positions). Since $f(x) \ge f(y)$, there are at most $\frac{\Delta}{2}$ zero-bits in $y$ in positions in $D$ (that is, $|y|_0^D \le \frac{\Delta}{2}$). Hence, if in $z$ more than half of bits in positions in $D$ are zeros (that is, $|z|_0^D > \frac{\Delta}{2}$), then $z$ has a strictly better fitness than $y$. Since $|z|_0^D$ follows a binomial distribution $\Bin(\Delta, \frac{1}{2})$, by Lemma~\ref{lem:binomial} we have
    \begin{align*}
        \Pr[|z|_0^D > \frac{\Delta}{2}] \ge \frac{1}{2}\left(1 - \sqrt{\frac{1}{\pi}}\right) \eqqcolon c_1',
    \end{align*}
    which is a candidate to be $c_1$ from the first lemma statement. Since $\Pr[f(z) > f(y) - \sqrt{k}] \ge \Pr[f(z) > f(y)]$, the same lower bound on this probability applies, hence $c_1'$ is also a candidate for $c_3$ from the third statement.

    \textbf{Case 2: $y$ is in the main slope, but not too close to the local optima.} More precisely, we assume that $|y|_1 \le n - k - 8\sqrt{k}$. In this case, $x$ has at least as many one-bits as $y$ (contrary to the previous case, since the slope points in the opposite direction). To create a better offspring $z$, we need to have strictly more one-bits than in $y$, but we also need to have $|z|_1 \le n - k$, since otherwise $z$ is in the fitness valley, and therefore its fitness is worse than the fitness of $y$ (note that we pessimistically disregard an unlikely event of $z$ being the global optimum). Therefore, the probability of $z$ being strictly better than $y$ is at least the probability of $z$ having more one-bits than $y$ minus the probability of $z$ having more than $n - k$ one-bits. The former can be bounded similar to the previous case, but swapping zero-bits and one-bits in all arguments, hence this probability is at least $c_1'$.

    We now bound the probability of $z$ having more than $n - k$ one-bits. We do it in two steps. First, we show that this is a sub-event of crossover taking no more than $\frac{\Delta}{2} + 2\sqrt{\Delta}$ one-bits in positions in $D$. Then, we bound the probability of this sub-event via Chernoff bounds.

    Let $\bar D$ be a complement of $D$ to $[n]$, that is, $\bar D = [n] \setminus D$. Note that $|x|_1^{\bar D} = |y|_1^{\bar D} = |z|_1^{\bar D}$ and that $|x|_1^D = \Delta - |y|_1^D$, hence we have
    \begin{align*}
        |x|_1 = |x|_1^D + |x|_1^{\bar D} = \Delta - |y|_1^D + |y|_1^{\bar D} = |y|_1 - 2|y|_1^D + \Delta,
    \end{align*}
    therefore,
    \begin{align*}
        \frac{|x|_1 + |y|_1}{2} = |y|_1 - |y|_1^D + \frac{\Delta}{2} = |y|_1^{\bar D} + \frac{\Delta}{2} = |z|_1^{\bar D} + \frac{\Delta}{2}.
    \end{align*}
    Using this, we can write
    \begin{align}
        \label{eq:z1}
        |z|_1 = |z|_1^D + |z|_1^{\bar D} = |z|_1^D + \frac{|x|_1 + |y|_1}{2} - \frac{\Delta}{2}.
    \end{align}
    Note that $x$ is not the global optimum, hence $|x|_1 \le n - k$. Also note that $\Delta$ is at most $|y|_0 + |x|_0 \le 2|y|_0$ (since in positions that are not zero in $y$ and that are not zero in $x$ they both have a one-bit, hence this position is not in $D$). Then, if we assume that $|z|_1^D \le \frac{\Delta}{2} + 2\sqrt{\Delta}$, then we have
    \begin{align}
        \label{eq:checkpoint}
        \begin{split}
            |z|_1 &\leq \frac{n - k + |y|_1}{2} + 2\sqrt{\Delta} \\
            &\le (n - k) - \frac{(n - k) - |y|_1}{2} + 2\sqrt{2|y|_0} \\
            &= (n - k) - \frac{|y|_0 - k}{2} + 2\sqrt{2|y|_0}.
        \end{split}
    \end{align}
    By solving a quadratic inequality, one can see that $\frac{|y|_0 - k}{2} \ge 2\sqrt{2|y|_0}$ if $|y|_0 \geq k+16 + 4\sqrt{2k+16}$. For $k \geq 100$ we have that $k+16 + 4\sqrt{2k+16} \geq k + 8\sqrt{k}$ and therefore $|y|_0 \geq k + 8\sqrt{k}$. Hence, from $|z|_1^D \le \frac{\Delta}{2} + 2\sqrt{\Delta}$ it follows that $|z|_1 \le n - k$. We only need to estimate the probability of the former event. Since $|z|_1^D$ follows a binomial distribution $\Bin(\Delta, \frac{1}{2})$ with expectation $\frac{\Delta}{2}$, the probability that $|z|_1^D$ exceeds $\frac{\Delta}{2} + 2\sqrt{\Delta}$ can be bounded via Chernoff bounds.
    \begin{align*}
        \Pr[|z|_1^D > \frac{\Delta}{2} + 2\sqrt{\Delta}] &= \Pr[|z|_1^D > E[|z|_1^D]\left(1 + \frac{4}{\sqrt{\Delta}}\right)] \\
        &\le \exp\left(-\frac{\left(\frac{4}{\sqrt{\Delta}}\right)^2 \frac{\Delta}{2}}{3}\right) = e^{-8/3}.
    \end{align*}
    Thus, if $k \ge 100$, then the probability that $|y|_1 < |z|_1 \le n - k$ is at least
    \begin{align*}
        c_1' - e^{-8/3} =  \frac{1}{2}\left(1 - \sqrt{\frac{1}{\pi}}\right) - e^{-8/3} \ge 0.148 \eqqcolon c_1''.
    \end{align*}
    
    Finally, we treat small values of $k < 100$. All arguments up to eq.~\eqref{eq:checkpoint} work for them as well, but $\frac{|y|_0 - k}{2} \ge 2\sqrt{2|y|_0}$ holds only when $|y|_0 \ge k + 16 + 4\sqrt{2k + 16}$ (which for $k < 100$ is not necessarily larger than $k + 8\sqrt{k}$). In these small distances from the global optimum we have $\Delta \le 2|y|_0 = O(1)$, hence the probability $\Pr[|z|_1^D = |y|_1^D + 1]$ is at least some constant (independent of $n$ and $k$). Denote this constant by $c_1'''$.

    \emph{Uniting cases 1 and 2.} We now have 
    \begin{align*}
        \Pr[f(z) > f(y)] \ge \min\{c_1', c_1'', c_1'''\} \eqqcolon c_1, 
    \end{align*}
    which completes the proof of the first statement. Note that $c_1$ also bounds $\Pr[f(z) > f(y) - \sqrt{k}]$ for all $y$ such that $f(y) \le n - 8\sqrt{k}$, hence it is our candidate for $c_3$.

    \textbf{Case 3: $y$ is close to the local optima.} Namely, in this case we assume that $n - k - 8\sqrt{k} < |y|_1 < n - k$. In this case, we utilize Lemma~\ref{lem:binom-lower}. Note that we have $|x|_0 \le |y|_0 \le k + 8\sqrt{k}$, hence the distance between them is at most $\Delta \le 2k + 16\sqrt{k} \le 18k$. In this case, to create a strictly better individual, it is sufficient to have $|z|_1^D = \lfloor \frac{\Delta}{2} \rfloor + 1$, since, similar to eq.~\eqref{eq:z1}, it implies
    \begin{align*}
        |z|_1 = \frac{|x|_1 + |y_1|}{2} + 1 - [\Delta \text{ is odd}]
    \end{align*}
    where $[\Delta \text{ is odd}] = 1$ if $\Delta$ is odd, and $0$ otherwise. If $|x|_1 = |y|_1$ (which implies that $\Delta$ is even), then $|z|_1 = |y|_1 + 1$, and since $y$ is not in the local optimum, $z$ does not fall to the fitness valley. Otherwise, this implies $|y|_1 < |z|_1 \le |x|_1$, hence $z$ is also better than $y$ and does no fall into the fitness valley.

    We now estimate the probability of $|z|_1^D = \lfloor \frac{\Delta}{2} \rfloor + 1$.
    If $\Delta$ is even, by Lemma~\ref{lem:binom-lower}, this probability is at least
    \begin{align*}
        \binom{\Delta}{\frac{\Delta}{2} + 1} 2^{-\Delta} = \binom{\Delta}{\frac{\Delta}{2} + \frac{1}{\sqrt{\Delta}} \cdot \sqrt{\Delta}}  2^{-\Delta} \ge \frac{1}{2\sqrt{\pi \Delta} e^{4/\sqrt{\Delta}}},
    \end{align*}
    and if $\Delta$ is odd, then it is at least 
    \begin{align*}
        \binom{\Delta}{\frac{\Delta}{2} + \frac{1}{2}} 2^{-\Delta} = \binom{\Delta}{\frac{\Delta}{2} + \frac{1}{2\sqrt{\Delta}} \cdot \sqrt{\Delta}}  2^{-\Delta} \ge \frac{1}{2\sqrt{\pi \Delta} e^{2/\sqrt{\Delta}}},
    \end{align*}
    Since $\Delta \le 18k$, in both cases this is at least 
    \begin{align*}
        \frac{1}{2\sqrt{\pi \Delta} e^4} \ge \frac{1}{2\sqrt{18\pi k}e^4} = \frac{c_2}{\sqrt{k}},
    \end{align*}
    where we define $c_2 \coloneqq \frac{1}{6\sqrt{2\pi}e^4}$.

    Also, in this sub-case, by Lemma~\ref{lem:binom-lower} for all $i \in [0..\lfloor\sqrt{\Delta}\rfloor]$ the probability that $|z|_1^D = \lceil \frac{\Delta}{2} \rceil - i$ is at least
    \begin{align*}
        \binom{\Delta}{\lceil \frac{\Delta}{2} \rceil - i} 2^{-\Delta} = \binom{\Delta}{\lceil \frac{\Delta}{2} \rceil - \frac{i}{\sqrt{\Delta}} \cdot \sqrt{\Delta}}  2^{-\Delta}  \ge \frac{1}{2\sqrt{\pi \Delta} e^{4i/\sqrt{\Delta}}} \ge \frac{1}{2\sqrt{\pi \Delta} e^4}. 
    \end{align*}
    Hence, recalling that $\Delta \le 18k$, we have
    \begin{align*}
        \Pr[f(z) > f(y) - \sqrt{k}] &\ge  \Pr[f(z) > f(y) - \sqrt{\frac{\Delta}{18}}] \\
        &\ge \sum_{i = 0}^{\lfloor \sqrt{\frac{\Delta}{18}} \rfloor} \Pr[|z|_1^D = \lceil \frac{\Delta}{2} \rceil - i] \\
        &\ge \sqrt{\frac{\Delta}{18}} \cdot \frac{1}{2\sqrt{\pi \Delta} e^4} = \frac{1}{6e^4\sqrt{2\pi}} \eqqcolon c_3'.
    \end{align*}

    Taking $c_3 \coloneqq \min\{c_1, c_3'\}$ completes the proof of the last statement.
\end{proof}

With Lemma~\ref{lem:better-offspring}, we can prove Lemma~\ref{lem:time-to-plateau}.

\begin{proof}[Proof of Lemma~\ref{lem:time-to-plateau}]
Enumerate the fitness values of $\jump$ by $a_0, a_1, \ldots , a_n$ in increasing order (that is, $a_n$ is the optimal value $n + k$ and $a_{n - 1}$ is the value at local optima $n$). For each $i \in \{0, \ldots, n\}$ and $j \in \{0, \ldots, \mu - 1\}$, we say that a population $P_t$ is in state $(i, j)$ if all individuals in $P_t$ have fitness at least $a_i$, and exactly $j$ individuals have fitness strictly greater than $a_i$. We aim to estimate the time until the algorithm reaches a population in a state $(i, j)$ with $i \geq n - 1$, which means that all individuals are either on the \plateau or are optimal. We call such states \emph{target states}. Note that the state of the population can only improve over time, where we define state $(i, j)$ to be better than state $(\ell, m)$ if $i > \ell$, or if $i = \ell$ and $j > m$ (that is, they are ordered lexicographically). Since states are attained in a non-decreasing order, each state is visited at most once, hence the time to reach a target state can be bounded by the sum of the times required to leave each non-target state. To estimate the time to leave a given state, we provide lower bounds on the probability of leaving that state in a single iteration.

We distinguish two cases: when the population consists only of copies of the same individual, and when there are several different genotypes in the population.

\textbf{Case 1: population of copies.} This is possible only in states $(i, 0)$. To escape this state the algorithm has to create a strictly better individual. It is possible by either performing a mutation-only iteration, where the algorithm flips one bit leading to a better state, or by a crossover iteration, where mutation after crossover does the same (note that the result of crossover will be identical to its parents, since the parents are identical). Therefore, for states $(i, 0)$ with $i < k - 2$ (representing the fitness valley), the probability of escaping in one iteration is the probability that mutation flips one of $(n - i + 1)$ one-bits (and does not flip any other bit), that is,
\begin{align*}
\frac{(n-(i+1))}{n} \left(1-\frac{1}{n}\right)^{n-1} \geq \frac{(n-k+1)}{en} \geq \frac{1}{2e}.
\end{align*}
Similarly, the probability to escape a state $(i, 0)$ where $i \ge k - 1$ is at least
\begin{align*}
\frac{(n - (i - (k - 1)))}{n} \left(1-\frac{1}{n}\right)^{n-1} \geq \frac{(n-i+k-1)}{en}.
\end{align*}

\textbf{Case 2: diverse population.} In this case, if the algorithm performs crossover and furthest parent selection, then it definitely chooses different individuals as parents. The algorithm escapes the current state, if it creates an individual that is strictly better than the worst individual in the population (since it either increases $i$ or $j$). We will use Lemma~\ref{lem:better-offspring} to estimate the probability to create such offspring. We consider two sub-cases depending on the worst fitness in the population. In both sub-cases we count on the algorithm to perform a crossover iteration with furthest parent selection (this occurs with probability $p_c\pfu$) and on mutation to flip no bits (this occurs with probability $(1 - \frac{1}{n})^n \ge \frac{1}{4}$).

\emph{Sub-case 2.1: $i < n - 9 \sqrt{k}$.} This implies that the worst fitness in the population is \emph{at most} $n - 9\sqrt{k}$. Then let $y$ be the worst of two parents chosen for crossover. If $f(y) \le n - 8\sqrt{k}$, then by Lemma~\ref{lem:better-offspring} (1) with probability $c_1$ we create an offspring that is better than $y$, hence it is better than the worst individual in the population. Otherwise, if $f(y) > n - 8\sqrt{k}$, then by Lemma~\ref{lem:better-offspring} (3) we create an offspring $z$ with fitness larger than $f(y) - \sqrt{k} > n - 9\sqrt{k} > i$ with probability at least $c_3$. Therefore, in this sub-case, the probability to escape the current state is at least $\frac{p_c\pfu}{4} \min\{c_1, c_3\}$.

\emph{Sub-case 2.2: $i \ge n - 9 \sqrt{k}$.} In this case, by Lemma~\ref{lem:better-offspring} (1) and (2), the probability to create an offspring that is better than the worst of the parents (and hence, its fitness is better than the worst fitness in the population) is at least $\min\{c_1, \frac{c_2}{\sqrt{k}}\}$. Therefore, probability to escape the current state is at least $\frac{p_c\pfu}{4} \min\{c_1, \frac{c_2}{\sqrt{k}}\}$

Now, if we denote by $p_{i,j}$ the probability to escape state $(i, j)$ in one iteration (which we estimated above), then the expected time until it happens is $\frac{1}{p_{i,j}}$. Summing these expected times over all states with $i < n - 1$, we can obtain on upper bound on time until the whole population is on the \plateau. We split this sum into several parts, and in the following we use $\frac{1}{\min(a, b)} \le \frac{1}{a} + \frac{1}{b}$ for multiple times. First, for $i < k - 2$ we have
\begin{align*}
    \sum_{i = 0}^{k - 3} &\Bigg(\frac{1}{1/(2e)}  + \sum_{j = 1}^{\mu - 1} \frac{1}{\frac{p_c\pfu}{4}\min\{c_1, c_3\}}\Bigg) \\
    &\leq \sum_{i = 0}^{k - 3} \left(2e  + \sum_{j = 0}^{\mu - 1} \frac{1}{\frac{p_c\pfu}{4}\min\{c_1, c_3\}}\right)  = O\left(k + \frac{k\mu}{p_c\pfu}\right).
\end{align*} 
Then, for $i \in [k-2..n-\lfloor 9\sqrt{k} \rfloor - 1]$ we have
\begin{align*}
    \sum_{i = k - 2}^{n - \lfloor 9\sqrt{k} \rfloor - 1} & \Bigg(\frac{1}{ \frac{n - i + k - 1}{en}} 
    + \sum_{j = 1}^{\mu - 1} \frac{1}{\frac{p_c\pfu}{4}\min\{c_1, c_3\}}\Bigg) \\
    &\leq \sum_{i = k - 2}^{n - \lfloor 9\sqrt{k} \rfloor - 1}  \Bigg(\frac{en}{(n - i + k - 1)} + \sum_{j = 0}^{\mu - 1} \frac{4}{p_c\pfu\min\{c_1, c_3\}}\Bigg) \\
    &= O\left(n \log(\frac{n}{k}) + \frac{n\mu}{p_c\pfu}\right).
\end{align*} 
Finally, for larger $i$ we have
\begin{align*}
    \sum_{i = n - \lfloor 9\sqrt{k} \rfloor}^{n - 2} & \Bigg(\frac{1}{ \frac{n - i + k - 1}{en}}
    + \sum_{j = 1}^{\mu - 1} \frac{1}{\frac{p_c\pfu}{4}\min\{c_1, \frac{c_2}{\sqrt{k}}\}}\Bigg) \\ 
    &\le \sum_{i = n - \lfloor 9\sqrt{k} \rfloor}^{n - 2} \Bigg(\frac{en}{(n - i + k - 1)} + \sum_{j = 0}^{\mu - 1} \frac{4}{p_c\pfu\min\{c_1, \frac{c_2}{\sqrt{k}}\}}\Bigg) \\
    &= O\left(n + \sqrt{k}\mu \cdot\frac{\sqrt{k}}{p_c\pfu}\right) = O\left(n + \frac{k\mu}{p_c\pfu}\right)
\end{align*}

Summing these three bounds, we complete the proof.
\end{proof}

We note that it is difficult to make a meaningful comparison between the bound from Lemma~\ref{lem:time-to-plateau} and the runtime of the vanilla \muoga with uniform parent selection, as there is limited prior work analyzing its behavior on such easy-to-climb landscapes. The closest related result for the vanilla \muoga is the analysis on the OneMax benchmark (essentially, \jump with $k = 1$) provided by~\citet{CorusO20} who proved a $O(n\log(n))$ runtime, but only for small values of $\mu = o(\sqrt{\log(n)})$. Our result indicates that the furthest parent selection with $p_c\pfu = \Omega(1)$ has a similar performance (in asymptotic sense) for such small population sizes. It is also notable that~\citet{Witt2006} showed $\Theta(n\log(n) + \mu n)$ runtime of the mutation-only \muoea, which is the same as our upper bound for the \muoga with furthest parent selection. This is good news for our algorithm, since it is known that some parent selection methods can slow down evolutionary algorithms. For example, the parent selection for the mutation-only \muoea proposed in~\cite{CorusLOW21} is slower on OneMax, solving it in $\Theta(\mu n \log(n))$ (despite this, that selection was shown to be very helpful at avoiding premature convergence to local optima).

Now we are able to combine Lemma~\ref{lem:time-to-plateau} with Theorem~\ref{thm:k-general} to derive the total runtime of the \muoga with furthest parent selection on \jump.

\begin{corollary}\label{cor:k-runtime-final}
Consider the \muoga with furthest parent selection on \jump for $2 \leq k \leq n/3$. Let $r \in [k-1]$. Suppose $\mu \ge c \cdot 4^r (k-r)\log(n)/(p_c \pfu)$ for a sufficiently large constant $c$. Then the expected number of fitness evaluations that the algorithm takes to find the global optimum is $O(n\mu /(p_c \pfu) + 4^r \cdot n^{k-r}/(p_c \pfu))$.
\end{corollary}

\begin{proof}
We need expected $O(n \mu/(p_c \pfu) + n \log(n))$ fitness evaluations until the whole population is on the \plateau by Lemma~\ref{lem:time-to-plateau}. Then, by Theorem~\ref{thm:k-general}, we need further $O(4^r \cdot n^{k-r}/(p_c \pfu))$ fitness evaluations in expectation to find the all-ones string. Summing up both runtimes provides the result, noting that the term $n \log(n)$ is absorbed by $n\mu/(p_c \pfu)$ since $k-r \geq 1$.
\end{proof}

For the smallest possible value of $\mu$ and for $r = k - 1$, this upper bound is simplified to $O(\frac{4^k n \log(n)}{(p_c\pfu)^2})$, which is $O(4^k n \log(n))$ when $p_c\pfu = \Omega(1)$. Comparing it to the $O(n^{k - 1})$ bound proved by~\citet{Doerr2024} for the vanilla \muoga, one can see that our bound is by factor $\Theta(\frac{(p_c\pfu)^2 n^{k - 2}}{4^k \log(n)})$ smaller. Interestingly, this significant performance boost is comparable with the boost of diversity mechanisms studied by~\citet{Dang2017} that do not allow to accept individuals that reduce diversity while breaking ties between individuals with equal fitness. 
However, \emph{in contrast with these mechanisms}, parent selection does not impose constraints on the search process by stricter offspring selection, but on the contrary it increases the variance of the crossover outcomes, allowing better exploration of the search space.

\section{Conclusion}

In this paper, we have performed runtime analysis of the \muoga with parent selection mechanism that prioritizes the most distant parents in iterations with crossover on the \jump benchmark. This analysis is powered by the introduced diversity measure which simultaneously captures the maximum distance $d_t$ between individuals in the population and the number of pairs $m_t$ of individuals in that distance. Our analysis revealed complex population dynamics that could be described thanks to this measure. Despite the diversity of the population can decrease sometimes, the proposed parent selection mechanism makes events of significant decrease (when $d_t$ decreases) very unlikely, so they rarely happen before we find the optimum. Interestingly, a similar picture was reported by~\citet{Doerr2024} in their experimental study of the vanilla \muoga: the maximum distance between individuals of population grew over time, and on each level of diversity the algorithm's position was solidified by creating more pairs of individuals in that distance.
Although the role of crossover in diversifying the population and in enhancing hill-climbing that we observed in this paper has been magnified by parent selection, that empirical result makes us believe that similar dynamics also arise in the vanilla \muoga. Thus, we are optimistic that the results of this paper also bring us closer to the understanding of the effectiveness of this classic GA and, more broadly, of population dynamics in crossover-based algorithms. 

\begin{acks}
  The authors thank Timo K{\"o}tzing, and Aishwarya Radhakrishnan for their contributions to the early stages of this project, including initiating the study and valuable discussions. We also acknowledge funding by the European Union (ERC, ``dynaBBO'', grant no.~101125586). 
\end{acks}

\bibliographystyle{ACM-Reference-Format}
\bibliography{references}

\end{document}